\documentclass[11pt]{article}

\usepackage[utf8]{inputenc}
\usepackage[colorlinks = true,
            linkcolor = blue,
            urlcolor  = blue,
            citecolor = blue,
            anchorcolor = blue]{hyperref}
\usepackage{amsmath}
\usepackage{amsthm}
\usepackage{amssymb}
\usepackage{xcolor}
\usepackage{algorithm,caption}
\usepackage{algorithmic}
\usepackage{cleveref}
\usepackage{graphicx}

\usepackage{natbib}
\usepackage{framed}
\usepackage{tcolorbox}

\usepackage{tikz}
\usetikzlibrary{arrows, fit, backgrounds, patterns, positioning, calc, intersections, decorations.markings, decorations.pathmorphing, decorations.pathreplacing, shapes.misc}

\usepackage{geometry}
\makeatletter
\theoremstyle{plain}
\newtheorem{theorem}{\protect\theoremname}
\newtheorem*{theorem*}{\protect\theoremname}
\newtheorem*{prop*}{\protect\theoremname}
\newtheorem{definition}{\protect\definitionname}
\theoremstyle{definition}

\theoremstyle{plain}
\newtheorem{lem}[definition]{\protect\lemmaname}
\newtheorem{claim}[definition]{\protect\claimname}
\newtheorem{cor}[definition]{\protect\corollaryname}
\newtheorem*{cor*}{\protect\corollaryname}
\newtheorem{remark}[definition]{\protect\remarkname}

\newtheorem{question}[definition]{\protect\questionname}
\newtheorem*{question*}{\protect\questionname}
\newtheorem*{assumption*}{\protect\assumptionname}

\makeatother

\newtheorem{innercustomthm}{Theorem}
\newenvironment{customthm}[1]
  {\renewcommand\theinnercustomthm{#1}\innercustomthm}
  {\endinnercustomthm}

\providecommand{\questionname}{Question}
\providecommand{\assumptionname}{Assumption}
\providecommand{\observationname}{Observation}
\providecommand{\corollaryname}{Corollary}
\providecommand{\definitionname}{Definition}
\providecommand{\lemmaname}{Lemma}
\providecommand{\claimname}{Claim}

\providecommand{\theoremname}{Theorem}
\providecommand{\exercisename}{Exercise}
\providecommand{\examplename}{Example}
\providecommand{\remarkname}{Remark}
\providecommand{\propname}{Proposition}

\newcommand{\vc}{\mathrm{VC}}
\newcommand{\R}{\mathbb{R}}
\renewcommand{\S}{\mathbb{S}}
\newcommand{\B}{\mathbb{B}}
\newcommand{\Q}{\mathbb{Q}}
\newcommand{\PP}{\mathbb{P}}

\newcommand{\X}{\mathcal{X}}
\renewcommand{\H}{\mathcal{H}}
\newcommand{\E}{\mathbb{E}}
\renewcommand{\P}{\mathcal{P}}
\newcommand{\A}{\mathcal{A}}

\newcommand{\sign}{\mathsf{sign}}

\newcommand{\ldim}{\mathtt{Ldim}}

\newcommand{\D}{\mathcal{D}}
\newcommand{\List}{\mathtt{LR}}
\newcommand{\N}{\mathbb{N}}
\newcommand{\TV}{\mathtt{TV}}
\newcommand{\F}{\mathcal{F}}
\newcommand{\G}{\mathcal{G}}
\newcommand{\oo}{\mathrm{o}}
\renewcommand{\L}{\mathcal{L}}
\newcommand{\dist}{\mathrm{dist}}
\newcommand{\maj}{\mathrm{maj}}

\author{Zachary Chase\footnote{Department of Mathematics, Technion. Supported by the European Union (ERC, GENERALIZATION, 101039692).} \and Bogdan Chornomaz\footnote{Department of Mathematics, Technion. Supported by the European Union (ERC, GENERALIZATION, 101039692).} \and Shay Moran\footnote{Departments of Mathematics, Computer Science, and Data and Decision Sciences, Technion and Google Research.
Robert J.\ Shillman Fellow; supported by ISF grant 1225/20, by BSF grant 2018385, by an Azrieli Faculty Fellowship, by Israel PBC-VATAT, by the Technion Center for Machine Learning and Intelligent Systems (MLIS), and by the European Union (ERC, GENERALIZATION, 101039692). Views and opinions expressed are however those of the author(s) only and do not necessarily reflect those of the European Union or the European Research Council Executive Agency. Neither the European Union nor the granting authority can be held responsible for them.
} \and Amir Yehudayoff}

\title{Local Borsuk-Ulam, Stability, and Replicability}

\begin{document}
\maketitle

\begin{abstract}
We use and adapt the Borsuk-Ulam Theorem from topology to derive limitations on list-replicable and globally stable learning algorithms. We further demonstrate the applicability of our methods in combinatorics and topology. 

We show that, besides trivial cases, both list-replicable and globally stable learning are impossible in the agnostic PAC setting. This is in contrast with the realizable case where it is known that any class with a finite Littlestone dimension can be learned by such algorithms. 
In the realizable PAC setting, we sharpen previous impossibility results and broaden their scope. 
%to a bigger class of algorithms.
Specifically, we establish optimal bounds for list replicability and global stability numbers in finite classes. This provides an exponential improvement over previous works and implies an exponential separation from the Littlestone dimension.
%and implies a tight lower bound in terms of the dual VC dimension. 
We further introduce lower bounds for weak learners, i.e., learners %with an error rate 
that are only marginally better than random guessing. %By comparison, 
Lower bounds from previous works apply only to stronger learners. 
%whose error is substantially smaller than a random guess.

%We believe that our adaptation of Borsuk-Ulam might find further use in computer science and combinatorics. Thus, 
To offer a broader and more comprehensive view of our topological approach, we %use it to prove a novel
prove a local variant of the Borsuk-Ulam theorem in topology and a result in combinatorics concerning Kneser colorings. In combinatorics, we prove that if \(c\) is a coloring of all non-empty subsets of \([n]\) such that disjoint sets have different colors, then there is a chain of subsets that receives at least \(1+ \lfloor n/2\rfloor\) colors (this bound is sharp).
In topology, we prove e.g.\ that for any open 
%(or finite closed) 
antipodal-free cover of the \(d\)-dimensional sphere, there is a point~\(x\) that belongs to at least \(t=\lceil\frac{d+3}{2}\rceil\) sets. %Further, if the cover contains both open and closed sets 
%(but not any set that is neither open nor closed) then there is a point \(x\) that belongs to at least \(\lceil t/2\rceil\) sets. Both bounds are sharp. 
\end{abstract}

\section{Introduction}

Topology is a %deep 
field in mathematics that studies the properties of spaces that remain unchanged under continuous transformations, such as convergence, compactness, and connectedness. By examining the invariants of a space %nature of shapes and spaces that can undergo 
under continuous transformations, % that do not ``tear it'', 
topology offers a lens through which the intrinsic structure of the space can be measured and understood. %Having its origins in geometry and set theory, 
Topology brings important ideas to many areas of science, and
%The influence of topology has since permeated a broad spectrum of scientific fields.
%\paragraph{Topology in Combinatorics and Computer Science.}
even though it focuses on the continuous, topology has furnished critical tools
even in discrete fields, like computer science and combinatorics.
%are predominantly concerned with discrete structures, topology s for a variety of areas.
In some sub-areas---like computational geometry---the contribution of topology is quite natural.
In some other sub-areas---like distributed computing, decision tree complexity and communication complexity--- the contribution is more surprising. 
A partial list of related papers includes \citep*{Monsky70, Lovasz78, Alon87, Kahn84, Smale87, Chaudhuri93, Borowsky93, Herlihy99, Saks00, Scheidweiler13, Hatami23}. For more details, we refer readers to the survey by~\cite{Bjorner96} and the book by~\cite{Matousek03BU}.

%More references (possibly less related) below
% \paragraph{Topology in Combinatorics and Computer Science.}
% While theoretical computer science and combinatorics are predominantly concerned with discrete structures, topology, which centers on the continuous, has furnished critical tools for a variety of areas. Some of these areas, like discrete and computational geometry, have a direct relationship with topology. However, there are domains such as circuit and decision tree complexity, communication complexity, regular languages, and distributed computing, where the applicability of topology might seem surprising.. A non-inclusive list of related papers includes \citep*{Monsky70, Lovasz78, Alon87, Benor83, Kahn84, Smale87, Chaudhuri93, Borowsky93, Pippenger97, Herlihy99, Saks00, Scheidweiler13, pellissier2021, Hatami23}. For more details we refer readers to the survey by~\cite{Bjorner96} and the book by~\cite{Matousek03BU}.

One of the most applicable results in topology is
the Borsuk-Ulam (BU) theorem~\citep{Borsuk1933}. 
%is one of the most applicable topological results in combinatorics and computer science. 
It is an extension of the intermediate point theorem to high dimensional space.
It deals with continuous ways to map
the $n$-dimensional sphere $\S^n \subset \R^{n+1}$
into Euclidean space. 
It says that mapping a sphere into a low
dimensional space leads to a collision 
between antipodal points. Namely,
for every continuous function \( f: \S^n \to \mathbb{R}^n \), there is a point \( x \) in \( \S^n \) such that \( f(x) = f(-x) \).

%\subsection*{Our Contribution}
In this work, we adapt the Borsuk-Ulam theorem to establish limitations of replicable and globally stable learning algorithms. To give a broader perspective of our topological approach, and to demonstrate its applicability, we apply it to derive a local variant of Borsuk-Ulam and a variant of the Lov\'{a}s-Kneser theorem in combinatorics. For convenience, our main results are denoted with letters: \Cref{t:localLS} to \Cref{t:finite}. All other theorems and claims are numbered.

\paragraph{Replicable Learning.}
Replicability is a foundational principle of the scientific method. A scientific study is deemed replicable if it consistently produces similar results when repeated under comparable conditions, even with new datasets. The exploration of replicable learning was initiated by \citet*{impagliazzo2022reproducibility} and further developed by \citet*{Bun23rep,Kalavasis23,Chase23rep,dixon2023list}. The works of \citet{Chase23rep,dixon2023list} notably utilized topological techniques, including variations of Sperner's lemma and the associated fixed point theorems. Building on their groundwork, our adaptation tackles several questions left open by these studies: 
\begin{itemize}
\item 
We first investigate the feasibility of globally stable and list-replicable algorithms in the agnostic PAC learning setting.
Our findings indicate that, besides trivial cases, both list-replicable and globally stable learning are unattainable (\Cref{t:repagn}). This sharply contrasts with the realizable PAC setting, where every class with a finite Littlestone dimension can be learned by such algorithms. Further amplifying this contrast is the fact that for other notions of algorithmic stability, such as differential privacy,
there is an equivalence between realizable and agnostic PAC learning. 

To circumvent the above impossibility result, we propose relaxed variants of agnostic list-replicability in which the size of the list can depend on the desired accuracy (\Cref{d:reperr} and \Cref{d:repopt}).

\item In the realizable PAC setting we derive impossibility results, offering significant quantitative improvements over prior work in terms of accuracy and list-replicability numbers. 

We show that the both the VC and dual VC dimensions lower bound the optimal list size achievable by weak learners, i.e.\ whose error is $<1/2$, only strictly better than a random guess (\Cref{t:weaklower}). This improves upon previous bounds by~\citet*{Chase23rep,dixon2023list} that apply to stronger learners whose error is $O(1/d)$, where $d$ is the VC dimension.

We also derive an optimal bound for the list-replicability and global stability numbers of finite classes (\Cref{t:finite}).
This yields an exponential improvement over~\citet*{Chase23rep}.
%For general classes, this implies an optimal lower bound in terms of the dual VC dimension. %(\Cref{cor:dualVC}).
We further use this bound to deduce that for every $d$ there is a class with Littlestone dimension $d$ for which the list-replicability number is at least exponential in $d$ (\Cref{cor:dualVC}); this nearly matches an upper bound by~\citet*{Ghazi21approx}.
%We present these results formally and discuss them in detail in Section~\ref{sec:resultsrep}.
\end{itemize}

\paragraph{Combinatorics and Topology.}
We believe that our adaptation of BU might find additional applications in theoretical computer science. Thus,
to provide a more comprehensive understanding of our technique and to illustrate its potential, 
we apply it to derive two additional results: 
%in topology and combinatorics:
\begin{itemize}
\item A local variant of the classical Borsuk-Ulam theorem in topology.
We prove that for any open (or finite closed) cover of the \(d\)-dimensional sphere that is antipodal-free, there exists a point \(x\) included in at least \(t=\lceil\frac{d+3}{2}\rceil\) sets. Moreover, if such a cover consists of both open and closed sets (but only of such sets) then there is a point \(x\) included in a minimum of \(\lceil t/2\rceil\) sets. Somewhat surprisingly, both of these bounds are shown to be sharp (see  \Cref{t:localLS}).

\item A combinatorial result related to Kneser colorings, which is inspired by the first demonstration of the topological method in combinatorics by~\cite{Lovasz78}. 
We prove that if \(c\) is a coloring of non-empty subsets of \([n]\) such that disjoint sets receive different colors, then there exists a chain of subsets that is assigned at least \(1+ \lfloor n/2\rfloor\) colors. This bound is sharp as witnessed by assigning some color $i\in A$ to every nonempty subset $A \subseteq [n]$ of size at most~$n/2$, and a separate color $c$ to all subsets of $[n]$ of size greater than $n/2$ (see \Cref{t:CC}).
\end{itemize}

%\paragraph{Organization.}
%In the following section, we provide a detailed exposition of our primary results; the remaining sections include the proofs.

%\subsection*{Our Contribution}

% \paragraph{Replicable and Globally Stable Learning.}
% Recently, \citet*{Chase23rep,dixon2023list} employed topological arguments in the realm of learning theory. In particular, they utilized fixed point theorems and variants of Sperner's lemma to establish impossibility results concerning replicable and globally stable learning.

% Informally, a learning rule is considered replicable if it consistently produces the same output when applied to two independent and identically distributed (i.i.d.) inputs. The study of replicability in PAC learning was initiated by \citet*{impagliazzo2022reproducibility}. A globally stable learning rule, on the other hand, consistently outputs the same hypothesis with a probability that is significantly greater than zero, though not necessarily nearing one. 

% The concept of global stability was introduced by \citet*{BunLM20} and was used by \citet*{BunLM20} and \citet*{Ghazi21approx} as an algorithmic tool to craft privacy-preserving learning rules. \cite{Chase23rep} further explored global stability as a variant of replicability and demonstrated its equivalence to \emph{list-replicability}, a notion introduced around the same time by \citet*{dixon2023list}.

% Informally define replicability, global stability, and list replicability.

% Describe the results by chase and dixon. 

\section{Main Results}
To provide a streamlined progression of our topological approach, we begin with our results in topology and combinatorics. Thereafter, we discuss our results in learning theory. Readers particularly interested in the learning-theoretic results may choose to skip directly to that section. 

\subsection{Topology}

The first result we present is a variant of Borsuk-Ulam theorem that demonstrates the local nature of our approach. 
This variant is most similar to the Lusternik-Shcnirelmann (LS) theorem, which is one of the equivalent formulations of Borsuk-Ulam \citep*{LS30}.\footnote{In fact, Lusternik-Shcnirelmann proved this variant in 1930, three years before Borsuk's publication.} 
%But before, we introduce some notation.
Let $\S^d=\{x\in\R^{d+1} : \|x\|_2 = 1\}$ denote the $d$-dimensional sphere. We say that $A\subseteq \S^d$ is antipodal-free if $A\cap(-A)=\emptyset$, where $-A=\{-x : x\in A\}$.
We say that a family $\mathcal{F}$ of (distinct) sets is antipodal-free if each $A\in\mathcal{F}$ is antipodal-free.
We say that $\mathcal{F}$ covers $\S^d$ if every $x\in\S^d$ belongs to some $F\in\mathcal{F}$.

\begin{theorem}[\citet*{LS30}]\label{t:LS}
Let $\mathcal{F}$ be a %finite 
antipodal-free cover of the $d$-dimensional sphere $\S^d$ such that
every set $A \in\mathcal{F}$ is either open or closed. Then, $\lvert \mathcal{F}\rvert \geq d+2$.
This is sharp---there is such a cover of size $d+2$.
\end{theorem}

Theorem~\ref{t:LS} says that to cover the sphere by open antipodal-free sets, we must use many sets. In the variant we develop, we imagine a large family of open antipodal-free sets $\mathcal{F}$ that cover the sphere. So, the bound $|\mathcal{F}| \geq d+2$ trivially holds. 
{\em Is there some local behavior that must hold? Can it be that every point in the sphere is covered only once?} Thinking of the circle $\S^1$, it seems that some points must be covered at least twice. {\em What happens in the $2$-dimensional sphere? Can every point be covered twice?}

%The subsequent result indicates the existence of a point $x \in \S^d$ that is a member of $\Omega(d)$ sets within the cover. To put it another way, even on a {\it local} scale, there exists a point 
%$x$ that witnesses that there are $\Omega(d)$ sets in the cover. 
%Quantitatively, the guaranteed point belongs to at least roughly $d/4$ sets:

\begin{definition}[Overlap-degree]
Define the \emph{overlap-degree} of a family $\F$ of sets as the maximal integer $k$ 
such that there exist $k$ sets $A_1,\ldots, A_k\in \F$ that overlap, i.e.\ $\cap_{i=1}^k A_i\neq\emptyset$.    
\end{definition}

\begin{customthm}{A}[local LS]\label{t:localLS}
Let $\mathcal{F}$ be a finite antipodal-free cover of the $d$-dimensional sphere $\S^d$.
\begin{enumerate}
\item If all sets in $\mathcal{F}$ are open then
the overlap-degree of $\F$ is at least $t := \lceil (d+3)/2\rceil$. 
\item If all sets in $\mathcal{F}$ are closed then
the overlap-degree of $\F$ is at least $t$. 
\item If all sets in $\mathcal{F}$ are either open or closed then the overlap-degree of $\F$ is at least $\lceil t/2 \rceil$.
\end{enumerate}

\noindent
All three bounds above are sharp
$($for the third item, there is an antipodal-free
cover with some sets open and some sets closed
with overlap-degree $\lceil t/2 \rceil)$.

%the order of $\F$ is at least $t = \lceil (d+3)/2\rceil$. 

%that is, there is a point $x\in\S^d$
%such that $x$ belongs to at least $t$ sets.
%This is sharp---there is such a cover of order $t$.

%The same is true if all sets in $\mathcal{F}$ are closed.

%If every set in $\mathcal{F}$ is either open or closed, then the order of $\F$ is at least $\lceil t/2 \rceil$. This is also sharp.
\end{customthm}
%Typically, both global and local LS theorems are applied to closed covers. Notably, 

In other words, if we try to cover the sphere with a finite number of open antipodal-free sets, then somewhere in the sphere
there must be a large overlap between the sets. 
As a curious observation, let us note that
the bound provided by the global LS theorem does not change if the sets in the cover are allowed to be either open or closed. The bound in the local version, however, drops by a factor of $\approx 2$. 
As a simple example, the 1-dimensional circle can be partitioned into four sets, each of which is open or closed
by taking two disjoint open half-circles and two (closed) singletons.

%\begin{proof}
%    COPY FROM AMIR'S DRAFT
%\end{proof}

Our proof of Theorem~\ref{t:localLS} relies on the BU theorem and a theorem by \cite{jan2000periodic}.
The sharpness is proved by two constructions
that are defined using barycentric sub-divisions of the simplex (the construction for the closed case immediately implies a construction for the open case). The second construction demonstrating the tightness of $\lceil t/2 \rceil$ is more intricate. It involves applying two consecutive barycentric sub-divisions. 
%triangulation twice; namely a barycentric triangulation of a barycentric triangulation of the simplex.

A slightly weaker bound in the special case when all the sets in the cover are open (or all the sets in the cover are closed) follows from a generalization of the BU that was proven by~\cite{Fan1952}. Fan's theorem is a generalization of Theorem~\ref{t:LS} which asserts the following. Let $A_1,\ldots, A_m$ be an antipodal-free open (or closed) cover of the sphere $\S^d$. Then, there are indices $i_1 < i_2< \ldots < i_{d+2}$ and a point $x\in \S^d$ such that \((-1)^jx\in A_{i_j}\) for all $j=1,2,\ldots d+2$.
It follows that either $x$ or $-x$ belongs to at least $\lceil (d+2)/2\rceil$ sets.
For some values of $d$, the bound
from Fan's theorem is off by a $-1$ from the true bound (on the overlap-degree of the cover).

%\begin{remark}
\begin{remark}\label{rem-topological-sphere}
Theorem~\ref{t:LS} and Theorem~\ref{t:localLS} are valid and typically used in a more general setup. 
The sphere $\S^d$ can be replaced by any topological space
that is homeomorphic to it, and the negation $x \mapsto -x$ can be replaced by a general continuous
involution $\nu$ (that is, $\nu^2$ is the identity map).
%Namely, $\S^d$ is a \emph{topological} $d$-dimension sphere, that is, a topological space homeomorphic to the standard $d$-dimensional sphere. And the negation map $x\mapsto-x$ is replaced by an arbitrary continuous fixed-point free involution $\nu:\S^d\to \S^d$. That is, $\nu(x)\neq x$ and $\nu(\nu(x)) = x$, for all $x\in\S^d$. Correspondingly, 
In this generality, 
a set $A$ is considered antipodal-free if \(A \cap \nu(A) = \emptyset,\) where $\nu(A)=\{\nu(x) : x\in A\}$.
%The standard $d$-dimensional sphere with the negation map is an example of such $\S^d$ and $\nu$. 
We assume this general setup throughout the paper when we use these theorems.
\end{remark}

We finish this section by proposing a possible direction for future research.
    It might be interesting to seek a local counterpart of the BU theorem (and not the LS theorem). 
    For example, consider a continuous $f:\S^d\to \mathbb{R}^n$ such that $f(x)\neq f(-x)$ for all $x\in\S^d$.
    Is there always a point $f(x)$ in the image of $f$ that has a large local dimension?
    By large local dimension we mean that every open neighborhood $U\subseteq\mathtt{Image}(f)$ of $f(x)$ has large dimension.
There are, of course, several reasonable options for a definition for ``dimension of $U$''.

\subsection{Combinatorics}

We now present an application in combinatorics inspired by the celebrated Lov\'asz-Kneser theorem~\citep{Lovasz78}. 
%BC: The rest of the paragraph is duplicated in the paragraph after the theorem.
%This is inspired by the celebrated Lov\'asz-Kneser Theorem~\citep{Lovasz78}, which is one of the pioneering applications of the topological method in combinatorics (arguably the first one).
%Given \( n \in \mathbb{N} \), 
A Kneser coloring \(c\) assigns to every non-empty subset $A\subseteq [n]$ a color $c(A) \in \N$ such that if $A \cap B = \emptyset$
then
$c(A)\neq c(B)$. 
%for any two disjoint subsets \( A, B \subseteq [n] %\).
Denote by ${[n] \choose k}$ the set of all $k$-element subsets of $[n]$.

\begin{theorem}[Lov\'asz-Kneser]\label{t:LK78}
Let $n$ and $k$ be positive integers with $n\geq 2k$. Any Kneser coloring of the sets in ${[n] \choose k}$ requires at least $n-2k+2$ distinct colors. This bound is optimal---there exists a Kneser coloring using precisely $n-2k+2$ colors on the sets in ${[n] \choose k}$.
\end{theorem}

Lov\'asz's proof for the lower bound in Theorem~\ref{t:LK78} utilizies the BU theorem, and is often regarded as pioneering the topological methods in combinatorics.  We refer the reader to Chapter 3.3 of~\cite*{Matousek03BU} for a short proof of Theorem~\ref{t:LK78} using Theorem~\ref{t:LS}.

We prove a related theorem using the local BU approach.
The Lov\'asz-Kneser theorem says that for every Kneser coloring, every antichain has many colors.
We prove that for every Kneser coloring, there is a chain with many colors. 
%Our subsequent theorem, proven using our localized adaptation of Borsuk-Ulam, establishes that every Kneser coloring of $\P([n])$ has a chain of subsets \(A_1\subseteq A_2\subseteq\ldots\subseteq [n]\) that receives many colors. This is, in a sense, dual to Theorem~\ref{t:LK78}, which lower bounds the number of colors in antichains of the form ${[n] \choose k}$. 

\begin{customthm}{B}[Colorful chains]\label{t:CC}
For any positive integer $n$ and a Kneser coloring $c$ of non-empty subsets of $[n]$, there exists a chain of subsets that receives at least $\lfloor n/2 \rfloor + 1$ distinct colors. That is,
there are $n$ distinct sets $A_1 \subset A_2 \subset \ldots \subset A_n \subseteq [n]$
so that $c$ has $\lfloor n/2 \rfloor + 1$ many colors
on them.
This bound is sharp---there are Kneser colorings that assign at most $\lfloor n/2 \rfloor + 1$ colors to all chains.
\end{customthm}
The upper bound in Theorem~\ref{t:CC} is simple. Color every nonempty subset $A \subseteq [n]$ of size at most~$n/2$ by some color $i\in A$, and assign a separate color to all subsets of $[n]$ of size greater than $n/2$. 
This coloring assigns no more than $\lfloor n/2 \rfloor + 1$ colors to any chain. Moreover, the total number of colors in this construction is $n+1$, which almost matches the trivial lower bound $n$ on the number of colors of any Kneser coloring.
The proof of the lower bound is more demanding, 
and follows similar lines to the proof of
%the proof being slightly more involved than the proof of 
Theorem~\ref{t:localLS}. 
%\new{IS IT?}, despite both utilizing the local Borsuk-Ulam idea.
%It will be interesting to find a combinatorial proof, even of a linear $\Omega(n)$ lower bound.

\smallskip

We conclude this section by posing a question for future research.
\begin{question}
For \( p \leq n \), a \( p \)-Kneser coloring \( c \) assigns a color \( c(A) \) to every non-empty subset \( A \subseteq [n] \) such that whenever \( p \) distinct sets \( A_1, \ldots, A_p \) have the same color, their intersection \( \cap_{i=1}^p A_i \) is non-empty. A \( 2 \)-Kneser coloring is simply a Kneser coloring. Consider a \( p \)-Kneser coloring which assigns a color \( i \in A \) to each subset \( A \subseteq [n] \) of size at most \( n(1-1/p) \) and a distinct color to subsets of \( [n] \) larger than this size. This way each chain receives at most \( \lfloor n \cdot (1-1/p) \rfloor + 1 \) colors. Theorem~\ref{t:CC} establishes this is optimal for \( p = 2 \). Is this the case for all \( p \leq n \)?
\end{question}

\subsection{Learning Theory}\label{sec:resultsrep}

In this section, we outline our findings in learning theory. 
%We start by introducing standard notation from the field, followed by a detailed presentation of our results.
We use standard notation from learning theory; we refer the reader to the book by~\cite{Shalev-Shwartz2014} for a detailed introduction.
Let $X$ be a set called the domain and let $Y=\{\pm 1\}$ denote the label set. A hypothesis/concept/classifier is a mapping $h:X\to Y$.
A concept class is a set of hypotheses $\H\subseteq Y^X$.  An example is an ordered pair $z=(x,y)\in X\times Y$. A sample $S$ is a finite sequence $S=(z_i)_{i=1}^n$ of examples.

A classification problem is defined by a distribution~$\D$ over examples. The learner does not know $\D$, but is able to collect a sample $S$ of i.i.d.\ examples from~$\D$
which she uses to build a classifier $h\colon X\to Y$.
Formally, a learning rule is a (possibly randomized) mapping $\A:(X\times Y)^\star \to Y^X$, where $(X\times Y)^*$ is the set of all samples (i.e.\ finite sequences of examples).
The population loss of an hypothesis $h$ with respect to a distribution $\D$, denoted $L_\D(h)$, is the probability that $h$ misclassifies a random example from $\D$, namely
$\L_\D(h)=\Pr_{(x,y)\sim \D}[h(x)\neq y]$. The empirical loss of an hypothesis $h$ with respect to a sample $S$, denoted $L_S(h)$, is the fraction of examples in $S$ that are misclassified by $h$, namely $\L_S(h)= \frac{1}{n}\sum_{i=1}^n 1[h(x_i)\neq y_i]$, where $S=(x_i,y_i)_{i=1}^n$.

%While the data distribution $\D$ is unknown to the learner, it seems that

Any effective %\textit{a priori} 
theory of learning should somehow
restrict the generality we are operating in.
A natural option is to set some limitations on the data distribution $\D$. The PAC model proposes 
to use an underlying concept class $\H\subseteq Y^X$,
which helps the analyst to reason about the problem.
%formulate assumptions about~$\D$.
The simplest assumption is that $\D$ is \emph{realizable},
that $\H$ contains hypotheses with arbitrarily small error (i.e.\ $\inf_{h \in \H} L_\D(h) = 0$).
%    With this assumption, the analyst's goal is to design a learning rule that outputs hypotheses with arbitrarily small error, given a sufficiently large sample.
    Another well-studied setting is the {\em agnostic} one in which we no longer assume that $\D$ is realizable,
    and accordingly we only require the learning rule to output hypotheses whose loss is competitive with $\L_\D(\H):=\inf_{h\in \H}\L_\D(h)$.
    
In both settings, the class $\H$ is used by the analyst to make algorithmic choices.
 In the realizable setting, the analyst 
 assumes that the data comes from the class $\H$.
 In the agnostic setting, the analyst does not assume anything on the input data,
 and the class $\H$ is thought of as the criterion of success; the output should be as good as any function in $\H$.

\paragraph{Replicability and Stability.}%\label{sec:resultsrep}
Replicability is a foundational principle of the scientific method. A scientific study is deemed replicable if it consistently produces similar results when repeated under comparable conditions, even with new datasets. 

Replicability and the related concept of global stability (formally defined below) have been recently introduced and studied within the context of machine learning theory. Global stability was first introduced by~\citet*{BunLM20} and was employed as an algorithmic tool for designing privacy-preserving learning rules by both \citet*{BunLM20} and \citet*{Ghazi21approx}. The exploration of replicability in PAC learning began with the work of \citet*{impagliazzo2022reproducibility}. \citet*{Chase23rep} studied global stability as a form of replicability and proved its equivalence with list-replicability, a concept introduced concurrently by~\citet*{dixon2023list}. \citet{dixon2023list} also defined the notion of {\em strong replicability} which is equivalent to global stability. 
%We next define global stability and list-replicability, which are our primary subjects of interest.

A learning rule $\A$ is called a $(\rho,\varepsilon)$-globally stable learner
for $\H$ if there exists$n$ so that for every distribution~$\D$ that is realizable by $\H$, there exists a predictor $h = h_\D$ such that $\L_{\D}(h) < \varepsilon$ and 
\[
\Pr_{S \sim \D^n}[\A(S) = h] \geq \rho.
\]
The rule $\A$ is called an agnostic $(\rho,\varepsilon)$-globally stable learner
for $\H$ if there exists $n = n(\rho,\varepsilon)$ so that for every distribution~$\D$ there exists a predictor $h = h_\D$ such that $\L_{\D}(h) < \L_{\D}(\H) + \varepsilon$ and 
$\Pr_{S \sim \D^n}[\A(S) = h] \geq \rho.$
% If $\A$ satisfies this property for all distributions $\D$ that are realizable by $\H$
% then we say $\A$ is an $(\rho,\varepsilon)$-globally stable learner for $\H$ in the realizable case.
%The global stability of a hypothesis class, defined next, is the largest stability parameter $\rho$ one can attain if the likely-outputted predictor $f$ (in \eqref{globally-stable}) has a small error with respect to $\D$.  
The class $\H$ is {\em learnable} with global stability number $\rho$ 
    if for every $\varepsilon>0$ there exists an $(\varepsilon,\rho)$-globally stable learner for $\H$.
Similarly, $\H$ is {\em agnostic} learnable with global stability number $\rho$ if
for every $\varepsilon>0$ there exists an agnostic $(\varepsilon,\rho)$-globally stable learner for $\H$.
    
\begin{definition}[Global stability number]
The global stability number $\rho(\H)$ of the class $\H$ is the supremum over all $\rho \in [0,1]$ for which $\H$ is learnable with global stability number $\rho$.
  We say that $\H$ is globally stable learnable if $\rho(\H) >0$.
The agnostic global stability parameter $\rho_{\mathtt{agn}}(\H)$ is defined analogously. 
    We say that $\H$ is agnostically globally stable learnable if $\rho_{\mathtt{agn}}(\H) >0$.
\end{definition}
Notice that if $\H$ is agnostically
globally stable learnable then it is also globally stable learnable and that $0 \leq \rho_{\mathtt{agn}}(\H) \leq \rho(\H)$.

A learning rule $\A$ is called $(\varepsilon,L)$-list replicable learner for $\H$
if for every $\delta>0$, there exists $n=n(\varepsilon,L,\delta)$ such that for every distribution $\D$ that is realizable by~$\H$,
there exist hypotheses $h_1,h_2,\ldots,  h_L$ such that
$$\Pr_{S\sim \D^n}[\A(S) \in\{h_1,\ldots, h_L\}] \geq 1-\delta$$
and 
for all $\ell \in [L]$,
$$\L_\D(h_\ell)< \varepsilon.$$
The rule $\A$ is called $(\varepsilon,L)$-list replicable {\em agnostic} learner for the class $\H$
if for every $\delta>0$, there exists $n=n(\varepsilon,L,\delta)$ such that for every distribution $\D$ there exist hypotheses $h_1,h_2,\ldots,  h_L$ such that
$\Pr_{S\sim \D^n}[\A(S) \in\{h_1,\ldots, h_L\}] \geq 1-\delta$ and $\L_\D(h_\ell)\leq \L_{\D}(\H) + \varepsilon$
for all $\ell \in [L]$.
We say that $\H$ is  learnable with replicability list size $L$
    if for every $\varepsilon>0$, there exists an $(\varepsilon,L)$-list replicable  learner for $\H$.
The definition of {\em agnostically} learnable with replicability list size $L$ is analogous. 

\begin{definition}[List replicability number]
    The list replicability number of $\H$ is defined as
    \[\List(\H) := \min\bigl\{L \in \N : \text{$\H$ is learnable with replicability list size $L$} \bigr\}.\]
    We say that $\H$ is replicable list learnable if $\List(\H) < \infty$.
The agnostic list replicability number $\List_{\mathtt{agn}}(\H)$ is defined analogously.
We say that $\H$ is agnostically replicable list learnable if $\List_{\mathtt{agn}}(\H) < \infty$.
\end{definition}
Observe that $\List(\H) \leq \List_{\mathtt{agn}}(\H) \leq \infty$ and thus if $\H$ is agnostically replicable list learnable then it is also replicable list learnable.

\begin{theorem}[\citet*{Chase23rep}]\label{t:rep-gs}
For every class $\H$,
\[ \List(\H)\cdot \rho(H)= 1 \quad\text{and}\quad \List_{\mathtt{agn}}(\H)\cdot \rho_{\mathtt{agn}}(H)= 1.\]
In particular, $\H$ is (agnostically) replicable list learnable if and only if it is (agnostically) globally stable learnable.
\end{theorem}
\begin{remark}
In the above theorem we use the convention that $\infty\cdot 0 = 1$, thus $\List(\H)=\infty$ if and only if~$\rho(H)=0$.
\end{remark}
Theorem~\ref{t:rep-gs} provides a quantitative equivalence between the list replicability and global stability numbers. We remark that the equivalence is algorithmic in the sense that globally stable algorithms can be efficiently converted to list replicable algorithms and vice versa. \cite{Chase23rep} prove Theorem~\ref{t:rep-gs} only in the realizable setting, however, the same proof applies in the agnostic setting.

\subsubsection{Agnostic PAC Learning}
The following theorem proved by \citet*{BunLM20} characterizes the classes that can be learned by a globally stable/list-replicable learning rule in the realizable setting.
The theorem relates $\rho(\H)$ to the Littlestone dimension  
$\ldim(\H)$ of $\H$.

\begin{theorem}[\citet*{BunLM20}]\label{t:gsld}
The following statements are equivalent for a concept class $\H$:
\begin{itemize}
    \item $\ldim(\H)<\infty$.\item $\rho(\H)>0$ (equivalently, $\List(\H)<\infty$). %that is, $\H$ is globally stable learnable (equivalently, list-replicable learnable).
\end{itemize}
\end{theorem}

Somewhat surprisingly, in the following result we show that Theorem~\ref{t:gsld} does not extend to the agnostic case. It says that if $\H$ is agnostically learnable by a globally stable/list-replicable algorithm then $\H$ must be finite.
In particular, list-replicable learnability in the realizable case does not imply list-replicable learnability in the agnostic case. This answers negatively a question posed by~\cite{Chase23rep}.

\begin{customthm}{C}[Agnostic replicability and global stability]\label{t:repagn}
The following statements are equivalent for a concept class $\H$:
\begin{itemize}
\item $\H$ is finite.
\item $\rho_{\mathtt{agn}}(\H)>0$ (equivalently, $\List_{\mathtt{agn}}(\H)<\infty$). 
%that is, $\H$ is agnostically globally stable learnable (equivalently, agnostically list-replicable learnable).
\end{itemize}
\end{customthm}
Every finite class $\H$ is trivially agnostically list replicable learnable by any agnostic proper learner (with the list being $|\H|$). The converse direction is proved using Theorem~\ref{t:localLS} and the notion of $\gamma$-interpolation from the work of~\citet*{AlonGHM23}. 

In other words, \Cref{t:repagn} asserts that agnostic list-replicable learning is possible only when it is trivially possible. Nevertheless, it might still be interesting to explore how the agnostic list-replicability number depends on the size of the class and on $n$ in the best and worst cases.

% \new{\Cref{t:repagn} implies a corollary regarding the amount of randomness agnostic replicable learners must use. The notion of replicable learning was introduced by~\cite{impagliazzo2022reproducibility} and is defined as follows.}
% \begin{definition}[Replicability~\citep{impagliazzo2022reproducibility}]
% \label{def:replicability}
% A learning rule $\A$ is called $(n,\rho)$-replicable if 
% for every distribution $\D$, it holds that
% \[
% \Pr_{S,S'\sim \D^n, r} [\A(S, r) = \A(S', r) ] \geq 1-\rho ,
% \]
% where $r$ is the internal randomness of $\A$, and $S,S'$ and $r$ are independent.

% A concept class $\H$ is replicably agnostically learnable if for every $\varepsilon$ and $\rho$ there exists an 
% $(n,\rho)$-replicable learner $\A$ such that for every distribution $\D$,
% \[\Pr_{S\sim \D^n}[L_\D(\A(S)\leq \varepsilonilon]\geq 1-\rho.\]
% \end{definition}
% It is known that every class $\H$ with a finite Littlestone dimension is replicably agnostically learnable~\cite{???}.

% \new{Observe that if $\A$ is an $(n,\rho)$-replicable learner that uses $R$ random bits in its internal random string $r$,
% then for every distribution $\D$ there is a list $L$ of at most $2^R$ hypotheses such that
% $\A(S,r)\in L$, with probability at least $1-\rho$ over $S\sim \D^n$ and $r\sim\{0,1\}^R$.}

% \medskip

Theorem~\ref{t:repagn} suggests that it could, perhaps, be beneficial to relax the definition of list replicable learnability in the agnostic setting. 
We propose two possible relaxations:
\begin{definition}[Excess-error dependent replicability]\label{d:reperr}
% We say that $\H$ is agnostically learnable with replicability list size $L=L(\varepsilon)$
%     if for every $\varepsilon>0$ there exists $L=L(\varepsilon)$ and an $(\varepsilon,L)$-list replicable agnostic learner for $\H$.
We say that $\H$ is agnostically learnable with excess-error dependent replicability
    if for every $\varepsilon>0$ there exists $L=L(\varepsilon)$ and a learning rule $\A$ with the following guarantees. For every $\delta>0$ there exists $n=n(\delta)$ such that for every distribution $\D$ there is a list of $L(\varepsilon)$ hypotheses $h_1,\ldots h_L$  such that $\L_{\D}(h_\ell)\leq \L_\D(\H) + \varepsilon$ for all $\ell$ and 
    \[\Pr_{S\sim \D^n}[A(S)\in \{h_1,\ldots, h_L\}]\geq 1-\delta.\]
\end{definition}

\begin{definition}[Class-error dependent replicability]\label{d:repopt}
We say that $\H$ is agnostically learnable with class-error dependent replicability
    if for every $\gamma>0$ there exists $L=L(\gamma)$ and a learning rule $\A$ with the following guarantees. For every $\varepsilon,\delta>0$ there exists $n=n(\varepsilon,\delta)$ such that for every distribution $\D$ for which $\L_\D(\H)\leq \gamma$ there is a list of $L(\gamma)$ hypotheses $h_1,\ldots h_L$  such that $\L_{\D}(h_\ell)\leq \L_\D(\H) + \varepsilon$ for all $\ell$ and 
    \[\Pr_{S\sim \D^n}[A(S)\in \{h_1,\ldots, h_L\}]\geq 1-\delta.\]
\end{definition}

What types of classes satisfy these definitions? Understanding the scope and nuances of these definitions could potentially shed light on the broader landscape of replicable learnability. We leave these questions for future work.

\subsubsection{Realizable PAC Learning}

We now turn to study quantitative bounds on the list-replicability and global stability numbers in the realizable case. For a class $\H$, the dual VC dimension of $\H$, denoted $\vc^*(\H)$, is the VC dimension of the dual class $\H^\star \subseteq \{\pm 1\}^\H$ of $\H$ defined by $x(h) = h(x)$.

\cite{Chase23rep} proved that $\List(\H)\geq \vc(\H)$ for every class~$\H$. In the proof, they showed that every list replicable algorithm that learns $\H$ with error $\varepsilon \leq O(1/d)$ must use list size of at least $d$. Their proof, however, does not provide any information in the case when the error of the algorithm is larger (e.g.\ a small constant). 
Likewise, the lower bound presented by \citet*{dixon2023list} also applies to learners whose error is $O(1/d)$.
This raises the question of whether it is possible to list replicably learn with error, say $\varepsilon=0.1$ or even $\varepsilon = 0.49$, and list size $\ll d$? The following result provides a negative answer, both in terms of $\vc(\H)$ and $\vc^\star(\H)$. We define $\List(\H, \varepsilon)$ to be the minimal $L$ for which there is an $(\varepsilon,L)$-list replicable learner of $\H$. In particular, $\List(\H, \varepsilon)$ is increasing in $\varepsilon$ and $\List(\H) = \lim_{\varepsilon \rightarrow 0} \List(\H, \varepsilon)$.

\begin{customthm}{D}[Lower bound for list replicable weak learners]\label{t:weaklower}
Let $\H$ be a concept class. Then, for any $0 < \varepsilon < 1/2$, it holds 
$\List(\H, \varepsilon)\geq \max\{ 1 + \lceil \vc(\H)/2\rceil, 1 + \lfloor \vc^\star(\H)/2\rfloor\}$.
\end{customthm}

As the $ \vc^*(\H)$ bound on $\List(\H, \varepsilon)$ is also a bound on $\List(\H)$, we have an easy corollary that applies to every class.
\begin{cor}\label{cor:dualVC}
Let $\H$ be a concept class, then $\List(\H) \geq \max\{\vc(\H), 1 + \lfloor \vc^\star(\H)/2\rfloor\}$. This bound is sharp in both parameters, that is, for all nontrivial values of $\vc(\H)$ and $\vc^\star(\H)$, there are classes whose list-replicability number is $\vc(\H)$, and those whose is $1 + \lfloor \vc^\star(\H)/2\rfloor$.
\end{cor}
The sharpness of the $\vc$ dimension bound can be found in \cite{Chase23rep}, and of the $\vc^*$ bound is from \Cref{t:finite} below. 

Note that by going from \Cref{t:weaklower} to \Cref{cor:dualVC}, that is, from learners that learn at least something to those that learn arbitrarily well, the $\vc^*$ bound remains the same (and sharp), but the $\vc$ bound increases by a factor of~$2$.
It would thus be interesting to determine whether the $\vc$ bound for weak learners can be improved and whether there are bounds for intermediate values of $\varepsilon$. 
However, it is worth noting that the sharpness of $\vc(\H)$ bound on $\List(\H)$ means that for the weak learners it also cannot exceed $\vc(\H)$.

Another interesting question is whether a bound on $\List(\H)$ also holds in the converse direction for finite classes; that is, whether $\List(\H)$ can be upper bounded by a function of $\vc(\H),\vc^\star(\H)$, provided that $\H$ is finite.\footnote{The assumption that $\H$ is finite is necessary as witnessed by the class of one-dimensional thresholds $\H=\{1[x\leq t]: t\in \R\}$. Here we have $\vc(\H)=\vc^\star(\H)=1$ but $\List(\H)=\infty$ because $\ldim(\H)=\infty$ (by Theorem~\ref{t:gsld}). Note however that every finite subclass
$\H'\subseteq \H$ has $\List(\H')=2$, as proven by~\cite{Chase23rep}.}

\Cref{t:weaklower} can be utilized to answer the most basic question about the list-replicability numbers for finite classes: let $\H$ be a finite concept class of size $\lvert \H\rvert = m$;
clearly, $\H$ is learnable by a list-replicable algorithm whose list-size is $m=\lvert \H\rvert$. Can this be improved? Is it possible to learn with a sublinear list size $o(m)$?
The following theorem gives a negative answer.

\begin{customthm}{E}[Finite classes]\label{t:finite}
For every class $\H$  of size $m$ we have
$\List(\H)\leq 1 + \lfloor m/2\rfloor$. 
This is sharp, as witnessed by the class $\H_m=\{h_i : i\leq m\}$ 
of $m$ projection functions on $X=\{0,1\}^m$: $h_i(x)=x_i$ for every $x\in X$ and $i$.\footnote{Here we use the label set $Y=\{0,1\}$, instead of $\{\pm1\}$ used in the rest of the paper.} For $\H_m$, $\vc^*(\H_m) = m$, and so $\List(\H_m) = \List(\H_m, \varepsilon)= 1 + \lfloor m/2\rfloor$, for any $0<\varepsilon<1/2$.  
\end{customthm}

The theorem says that for all finite classes, it is possible to improve over the trivial list size of $\lvert \H\rvert$ by a factor of $\approx \tfrac{1}{2}$.
But it also says that in some cases this is the only improvement possible. 
In particular, since the Littlestone dimension of a finite class $\H$ satisfies $\ldim(\H) \leq \log\lvert \H\rvert$, we get:
\begin{cor}\label{cor:Ldim}
For every $d$, there exists a class $\H$ with $\ldim(\H)=d$ such that for every $\varepsilon<1/2$, $\List(\H, \varepsilon)\geq 2^{d-1}$. In particular, $\List(\H)\geq 2^{d-1}$.
\end{cor}
Thus, despite the fact that every Littlestone class is list-replicable learnable, there are cases where the list-replicability number is exponential in the Littlestone dimension.
It is worth noting that the best known upper bound is $\List(\H)\leq \exp(\exp(\ldim(\H)))$,
by~\cite{BunLM20}. Relatedly, \citet*{Ghazi21approx} provided an upper bound on the list-size 
which depends only exponentially on the Littlestone dimension. Namely, that for a class $\H$ with $\ldim(\H) = d$ and for a fixed $\varepsilon>0$, $\List(\H, \varepsilon) \leq O(d~2^{O(d^2)})$. However, the dependency on $\varepsilon$ in their bound cannot be easily removed to yield a bound on $\List(\H)$. Corollary~\ref{cor:Ldim} then implies that an exponential dependence of $\List(\H, \varepsilon)$ on the Littlestone dimension is necessary, even for large $\varepsilon<1/2$.
%It remains open to determine whether $\List(\H)\leq \exp(\ldim(\H))$ for every class $\H$.

%\subsubsection{Online Learning}\label{sec:resultonline}

\section{Proof of \Cref{t:localLS}}

\subsection{Normal Spaces}\label{sec-normal-spaces}	

We start with a preliminary discussion on topological spaces. 
Let $X$ be a normal topological space;
that is, for every two disjoint closed sets $C_1, C_2$
there are two disjoint open sets $O_1, O_2$ so that
$C_1 \subset O_1$ and $C_2 \subset O_2$.

\begin{claim}
Every metric space is a normal topological space.
\end{claim}
Indeed, this elementary claim follows by replacing every point in $C_1$ and $C_2$
by a sufficiently small ball such that the balls around points in $C_1$ are disjoint from the balls
around points in $C_2$.

The following stronger property is equivalent to normality.
The interior of $A$ is denoted by $A^\oo$.

\begin{claim}
\label{clm:obsNormal}
If $X$ is normal, then 
for every disjoint closed $C_1,C_2$,
there are disjoint closed $F_1,F_2$ so that
$C_1 \subset F_1^\oo$ and $C_2 \subset {F}_2^\oo$.
\end{claim}

The claim above follows from the following claim
(applied twice to $C_1 \subset O_1$ and $C_2 \subset O_2$).

\begin{claim}
If $C \subset O$ for $C$ closed and $O$ open
then there is $F$ closed so that
$C \subset F^\oo \subset F \subset O$.
\end{claim}

\begin{proof}
The sets $C$ and $X \setminus O$ are disjoint and closed.
Let $B,B'$ be disjoint open sets so that $C \subset B$
and $X \setminus O \subset B'$.
Let $F$ be the closure of $B$.
Because $X \setminus B'$ is closed and contains $B$,
we can deduce that $F \subset X \setminus B'$.
Because $X \setminus O \subset B'$,
we know that $X \setminus B' \subset O$.
Because $B$ is open,
we know that $C \subset B \subset F^\oo$.
\end{proof}

For an integer $m \geq 2$, a topological space is $m$-normal if
for every collection of closed sets
$C_1,\ldots,C_m$,
there is a collection of closed sets 
$F_1,\ldots,F_m$ so that for all $i \in [m]$,
$$C_i \subset F^\oo_i$$
and for every $S \subset [m]$
$$\bigcap_{i \in S} F_i \neq \emptyset
\ \Rightarrow \ 
\bigcap_{i \in S} C_i \neq \emptyset.$$
A topological space is $\infty$-normal if it is $m$-normal for every integer $m \geq 2$. 
Note that every $m$-normal space (for $m \geq 2$) is normal and being $2$-normal is equivalent to being normal. 
The following lemma shows that being normal is in fact equivalent to being $\infty$-normal.

\begin{lem}
\label{lem:normal}
Every normal space is $\infty$-normal.
\end{lem}

\begin{proof}
Let $C_1,\ldots,C_m$ be closed sets.
It will be convenient to set $C_0 = \emptyset$.
We claim, by induction, that for each $j \in [m]$,
there are closed sets $F_0,F_1,\ldots,F_j$ so that
$C_i \subset F_i^\oo$ for all $i \in \{0,1,\ldots,m\}$
and so that if the intersection of a subset of $F_0,F_1,\ldots,F_j,C_{j+1},\ldots,C_m$
is nonempty then the corresponding intersection of
$C_0,C_1,\ldots,C_m$ is also nonempty.
Applying this claim for $j=m$ completes the proof.

The base case $j=0$ is trivial with $F_0 = C_0$.
For the step, let $0<j<m$ and
assume we already chose $F_1,\ldots,F_j$.
Let $A_i = F_i$ for $i \leq j$,
and let $A_i = C_i$ for $i>j+1$.
For every non empty $S \subset [m] \setminus \{j+1\}$
so that $\bigcap_{i \in S} A_i \neq \emptyset$
and $C_{j+1} \cap \bigcap_{i \in S} A_i  = \emptyset$,
let $H_S$ be a closed set so that $C_{j+1} \subset H_S^\oo$
that is disjoint from $\bigcap_{i \in S} A_i$.
When $\bigcap_{i \in S} A_i = \emptyset$,
set $H_S = X$.
So,
\begin{align*}
 H_S \cap \bigcap_{i \in S} A_i \neq \emptyset \ \
\Rightarrow \ \  C_{j+1} \cap \bigcap_{i \in S} A_i \neq \emptyset .
\end{align*}
Define the closed set $F_{j+1}$ to be
$$F_{j+1} = \bigcap_S H_S.$$
Because $C_{j+1} \subset H_S^\oo$ for all $S$, we have $C_{j+1} \subset F_{j+1}^\oo$. 
To complete the proof, 
assume $T \subset \{0,\ldots,m\}$ is so that $j+1 \in T$ and
$\bigcap_{i \in T} A_i \neq \emptyset$,
where $A_{j+1} = F_{j+1}$. 
Because $F_{j+1} \subset H_{T \setminus  \{j+1\}}$,
\begin{align*}
H_{T \setminus  \{j+1\}} \cap \bigcap_{i \in T \setminus \{j+1\}} A_i \neq \emptyset 
\ \Rightarrow \  C_{j+1} \cap \bigcap_{i \in T \setminus \{j+1\}} A_i \neq \emptyset .
\end{align*}
By the induction hypothesis, we can conclude
\begin{align*}
C_{j+1} \cap \bigcap_{i \in T \setminus \{j+1\}} A_i \neq \emptyset
\ \Rightarrow \  \bigcap_{i \in T} C_i \neq \emptyset .
\end{align*}
Altogether,
\begin{equation*}
\bigcap_{i \in T} A_i \neq \emptyset
\ \Rightarrow \  \bigcap_{i \in T} C_i \neq \emptyset . \qedhere
\end{equation*}

\end{proof}

\subsection{The Lower Bounds}

The proof relies on the following theorem by \citet*{jan2000periodic}.
%Periodic Coincidence for Maps of Spheres, Kobe Journal of Math. 17 (2000), 21-26. MR 2001k:55007
\begin{theorem}
\label{thm:Jaw}[\cite{jan2000periodic}]
Let $P$ be a finite 
simplicial complex of dimension $d$
that is realized in Euclidean space.
If $f : S^n \to P$ is continuous and $2d \leq n$ 
then there is $x \in S^n$ so that
$f(x) = f(-x)$. 
\end{theorem}

The proof of Theorem~\ref{thm:Jaw} under the weaker assumption that $2d+1 \leq n$
follows from the following two known results.
The first is the realization theorem, first proved by Menger, stating that $P$ can be realized in $\R^{2d+1}$.
The second is the Borsuk-Ulam theorem stating that 
every continuous map from $S^n$ to $\R^n$ must collapse antipodal points. 
%The proof of the theorem above appears in \new{Citation. Although, we have already mentioned Jaworowski}.
We partition the proof into two cases.

\begin{proof}[Lower bound in open case]
Assume that $A_1,\ldots,A_k$ are open subsets of $S^n$
that are antipodal-free and cover the sphere.
Define a continuous map $\tilde g : S^n \to \R^k$ as follows: for each $x \in S^n$
and $i \in [k]$,
$$\tilde g_i(x) = \dist(x, S^n \setminus A_i),$$
where $\dist(x,A) = \inf \{ \|x - y\|_2 : y \in A\}$ for $A \subset S^n$.
Define
$$g(x) = \frac{\tilde g(x)}{\|\tilde g(x)\|_1}.$$
Let $\ell = \max \{ L(x, \{A_1,\ldots,A_k\}) : x \in S^n\}$.
Each $g(x)$ has sparsity at most $\ell$ and is 
in a co-dimension one (affine) subspace.
Hence,
the image of $g$ is contained in an $(\ell-1)$-dimensional polyhedron. 
In addition, we claim that for all $x \in S^n$,
we have $g(x) \neq g(-x)$.
Indeed, for each $x$, let $i \in [k]$ be so that $x \in A_i$ and $-x \not \in A_i$.
Because $A_i$ is open,
$\tilde g_i(x) >0$ and $\tilde g_i(-x) = 0$.
So, $g(x) \neq g(-x)$. 
Theorem~\ref{thm:Jaw}, therefore, implies that
$2(\ell-1) \geq n+1$.
\end{proof}

\begin{proof}[Lower bound in open/closed case]
The proof is by reduction to the open case.
Assume that $A_1,\ldots,A_k$ are open/closed antipodal-free sets that cover $S^n$,
where not all are open. 
Let $C \subset [k]$ be the set of $i \in [k]$ so that $A_i$ is closed. 
Let $\{F_i : i \in C\}$ be the collection of closed sets given by Lemma~\ref{lem:normal}.
For each $i \in C$, let $O_i$ be an open set so that $A_i \subset O_i$
and $O_i$ is antipodal-free.
For each $i \in C$, consider the open set
$$A'_i = F_i^\oo \cap O_i.$$
For $i \not \in C$, let $A'_i = A_i$.
We obtained a new open antipodal-free cover of the sphere
$A'_1,\ldots,A'_k$.
It remains to control the cover numbers. 
Let $x$ be so that $L(x,\{A'_1,\ldots,A'_k\}) \geq L_0(n)$.
Let $C'_x = \{i \in C : x \in A'_i\}$.
If $\bigcap_{i \in C'_x} A'_i = \emptyset$ then we are done.
Otherwise, by choice of $\{F_i : i \in C\}$,
$$\bigcap_{i \in C'_x} A'_i \neq \emptyset
\ \Rightarrow \ \bigcap_{i \in C'_x} F_i \neq \emptyset
\ \Rightarrow \ \bigcap_{i \in C'_x} A_i \neq \emptyset.$$
Let $y \in \bigcap_{i \in C'_x} A_i$.
So,
$$L_0(n) = L(x,\{A'_1,\ldots,A'_k\}) \leq 
L(x,\{A_i : i \not \in C \}) + 
L(y,\{A_1,\ldots,A_k\}).$$
When all sets $A_i$ are closed,
the term $L(x,\{A_i : i \not \in C \})$ is zero.
\end{proof}

\subsection{Upper Bound}

\Cref{t:localLS} is stated for the $d$-dimensional sphere $\S^d$ with the involution $x \mapsto -x$. However, as stated in \Cref{rem-topological-sphere}, we can consider any topological space homeomorphic to $\S^d$ with an arbitrary fixed-point free involution $\nu$. For this purpose, it will be convenient to consider the barycentric subdivision of the boundary of a $(d+1)$-dimensional simplex.

\begin{definition}[Barycentric subdivision of the boundary of a simplex]\label{def-barycentric}
The barycentric subdivision $\B^d$ of the boundary of the $(d+1)$-dimensional simplex $\Delta_{d+1}$ is defined as follows.
The vertices of $\Delta_{d+1}$ are the elements of $[d+2]$.
The vertices of $\B^d$ are all the non-trivial subsets $T$ of $[d+2]$, that is, all subsets except for $\emptyset$ and $[d+2]$.
The simplices of $\B^d$ are chains $\sigma = \{T_1 \subsetneq T_2 \subsetneq \ldots \subsetneq T_t\}$ of nontrivial subsets of $[d+2]$.

We have the following geometric implementation in mind.
There are $d+2$ points in general position in $\R^{d+1}$ that are called $1,\ldots,d+2$.
Each non-trivial subset $T$ of $[d+2]$ is the average
of the points in $T$.
A point $x$ in a simplex $\sigma = \{T_1 \subsetneq \ldots \subsetneq T_t\}$ of $\B^d$
can be uniquely expressed as a convex combination $x = \sum_i \lambda_{T_i} T_i$.
For each $T$, we denote by $\lambda_T(x)$
the coefficient of $T$ in the expression above, 
if $T \not \in \sigma$ then $\lambda_T(x)=0$.

Antipodality is defined via a continuous involution $\nu$.
The involution is defined on the vertices by $\nu(T) = [d+2] \setminus T$.
It is defined over all of $\B^d$ by linear interpolation. 
That is, for all $x$ and $T$, we have
$\lambda_T(\nu(x)) = \lambda_{\nu(T)}(x)$.

Thus defined, $\B^d$ is homeomorphic to $\S^d$, and $\nu$ is a continuous fixed-point free involution over it. The concept is illustrated in \Cref{fig-BARYCENTRIC} below.
\end{definition}

\begin{remark}
The involution $\nu$ of $\B^d$ is equivalent to the standard antipodality map $x \mapsto -x$ of $\S^d$, in the sense that the former is obtained from the latter by conjugating it by a homeomorphism from $\B^d$ to $\S^d$. We note that, in general, topological spheres admit fixed-point free involutions that are not standard~\citep{hirsch64}. Although we will not use those, \Cref{t:localLS} is still applicable in this situation.
\end{remark}

\begin{figure}[!hbt]
	\centering
	\begin{tikzpicture} 
[
pt/.style={inner sep = 0.0pt, circle, draw, fill=black},
point/.style={inner sep = 1.7pt, circle,draw,fill=white},
spoint/.style={inner sep = 1.2pt, circle,draw,fill=white},
mpoint/.style={inner sep = 1.7pt, circle,draw,fill=black},
ypoint/.style={inner sep = 3pt, circle,draw,fill=yellow},
xpoint/.style={inner sep = 3pt, circle,draw,fill=red},
FIT/.style args = {#1}{rounded rectangle, draw,  fit=#1, rotate fit=45, yscale=0.5},
FITR/.style args = {#1}{rounded rectangle, draw,  fit=#1, rotate fit=-45, yscale=0.5},
FIT1/.style args = {#1}{rounded rectangle, draw,  fit=#1, rotate fit=45, scale=2},
vecArrow/.style={
		thick, decoration={markings,mark=at position
		   1 with {\arrow[thick]{open triangle 60}}},
		   double distance=1.4pt, shorten >= 5.5pt,
		   preaction = {decorate},
		   postaction = {draw,line width=0.4pt, white,shorten >= 4.5pt}
	}
]

\begin{scope}[yscale=1.2, xscale = 0.7]
	\begin{scope}[yscale = 0.7]
		\node(1) at (0,2) {$1$};
		\node(2) at (2,-1) {$2$};
		\node(3) at (-2,-1) {$3$};
	
		\node(12) at (2,1) {$12$};
		\node(23) at (0,-2) {$23$};
		\node(13) at (-2,1) {$13$};

		\draw (1)--(12)--(2)--(23)--(3)--(13)--(1);
	\end{scope}
\end{scope}

\end{tikzpicture}	
	\caption{The simplicial complex $\B^1$
 is the barycentric subdivision of the boundary of a triangle (with vertices $1,2,3$). It is homeomorphic to the sphere $\S^1$. The simplices are chains (like $1 \subsetneq 12$ and $3 \subsetneq 23$). The involution $\nu$ swaps the vertices $1$ and $23$, the vertices $2$ and $13$, and the vertices $3$ and $12$. The involution in this case is rotation by $180^o$ around the center. }
	\label{fig-BARYCENTRIC}       
\end{figure}
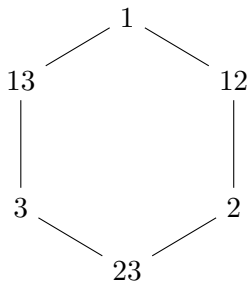

\begin{proof}[Upper bound in closed case]
We need to define the closed cover.
Let $t = \lceil (d+3)/2 \rceil$.
For each $i \in [d+2]$, we will have a closed set $A_i$,
and there is one additional set $A_+$.
Define the following weight function of subsets $T$ of $[d+2]$:
$$w(T) = \begin{cases}
1 & |T| < t ,\\
1/2 & |T| \geq t .
\end{cases}$$
For a point $x$, define
$$h(x) = \max \{ w(T) \lambda_{T}(x) : T \} >0.$$
For $T$ of size $|T|< t$, defined the closed set
$$F_T = \{x \in \B^d  : h(x) = w(T) \lambda_T(x) \}.$$
For $i \in [d+2]$,
define
$$A_i = \bigcup_{T: i \in T , |T| < t} A_T.$$
The closed set $A_+$ is
$$A_+ = \{x \in \B^d  : \exists T \ |T|\geq t  \ \wedge  h(x) = w(T) \lambda_T(x) \}.$$
All sets thus defined are closed, and they form a cover of the sphere. 
Each $x$ belongs to at most $t = t-1+1$ sets,
because the support $\sigma = \{T_1 \subsetneq \ldots \subsetneq T_t\}$ of $x$ is a chain
and $x \in A_i$ only when $i \in T_j$ of size $|T_j|<t$ for some $j$.
It remains to verify the antipodal-free condition. 

If both $x$ and $\nu(x)$ are in $A_i$,
then there are subsets $T, T'$ of size at most $t$
such that $i\in T, T'$, $\lambda_T(x) >0$,
and $\lambda_{T'}(\nu(x)) = \lambda_{\nu(T')}(x) > 0$.
Because simplices are chains, 
it follows that $T \subsetneq \nu(T')$.
We get a contradiction: $i \in \nu(T')$ implies $i \not \in T'$.

If both $x$ and $\nu(x)$ are in $A_+$ then
there are $T_1$ and $T_2$ of sizes
$|T_1| \geq t$ and $|\nu(T_2)| \geq t$ so that
$$h(x) = w(T_1) \lambda_{T_1}(x) = (1/2) \lambda_{T_1}(x) \geq 
w(T_2) \lambda_{T_2}(x) = \lambda_{T_2}(x) $$
and
$$h(\nu(x)) = w(\nu(T_2)) \lambda_{\nu(T_2)}(\nu(x)) = (1/2) \lambda_{T_2}(x)
\geq  w(\nu(T_1)) \lambda_{\nu(T_1)}(\nu(x))
=  \lambda_{T_1}(x).$$
Again, we get a contradiction:
$$\lambda_{T_1}(x) \geq 
2 \lambda_{T_2}(x) \geq 4 \lambda_{T_1}(x).$$
\end{proof}

%\begin{proof}[Upper bound in closed case]
%We need to define the closed cover.
%Let $d = \lceil (n+3)/2 \rceil$.
%For each $T$ of size $|T| < d$ we will have a closed set
%$A_T$, and there is one additional set $A_+$.
%Define the following weight function:
%$$w(T) = \begin{cases}
%1 & |T| < d ,\\
%1/2 & |T| \geq d .
%\end{cases}$$
%For a point $x$, define
%$$h(x) = \max \{ w(T) \lambda_{T}(x) : T \} >0.$$
%The closed set $A_T$ for $T$ of size $|T|< d$ is
%$$A_T = \{x \in S^n  : h(x) = w(T) \lambda_T(x) \}.$$
%The closed set $A_+$ is
%$$A_+ = \{x \in S^n  : \exists T \ |T|\geq d  \ \wedge  h(x) = w(T) \lambda_T(x) \}.$$
%All sets thus defined are closed, and they form a cover of the sphere. 
%Each $x$ belongs to at most $d = d-1+1$ sets.
%It remains to verify the antipodal-free condition. 
%If $x$ is in $A_T$ for $|T|<d$,
%then $\nu(x)$
%is not covered by $A_{T}$ because 
%$\lambda_{\nu(T)}(x)=0$.
%Finally, assume towards a contradiction that
%both $x$ and $\nu(x)$ are in $A_+$.
%This means that there are $T_1$ and $T_2$ of sizes
%$|T_1| \geq d$ and $|\nu(T_2)| \geq d$ so that
%$$w(T_1) \lambda_{T_1}(x) \geq w(T_2) \lambda_{T_2}(x)
%\quad \text{and} \quad w(\nu(T_2)) \lambda_{T_2}(x)
%\geq w(\nu(T_1)) \lambda_{T_1}(x).$$
%By the definition of $w(\cdot)$ we get a contradiction. 
%\end{proof}

\begin{remark}
A construction of open sets can be obtained
from the closed sets above using Lemma~\ref{lem:normal}.
\end{remark}

\begin{proof}[Upper bound for open/closed case]
For this construction we need to perform a second barycentric subdivision.
Let $\Q^d$ be the barycentric subdivision of $\B^d$.
The vertices of $\Q^d$ are increasing chains of non-trivial subsets of $[d+2]$.
A typical vertex in $\Q^d$ is of the form $v = \{T_1 \subsetneq \ldots \subsetneq T_r\}$.
The geometric realization of $v$ is as
$$v =  \sum_q \frac{c}{w(T_q)} T_q,$$
where $w(\cdot)$ was defined above
and $c$ is the constant so that $c \sum_q \frac{1}{w(T_q)} = 1$.
The simplices of $\Q^d$ are chains of chains.
A typical simplex is of the form $\sigma = \{v_1 \subsetneq \ldots \subsetneq v_m\}$. This construction is illustrated in \Cref{fig-DOUBLE-BARYCENTRIC} below.

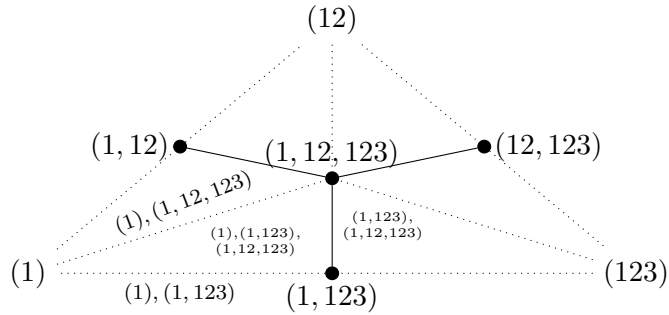
\begin{figure}[!hbt]
	\centering
	\begin{tikzpicture} 
[
pt/.style={inner sep = 0.0pt, circle, draw, fill=black},
point/.style={inner sep = 1.7pt, circle,draw,fill=white},
spoint/.style={inner sep = 1.2pt, circle,draw,fill=white},
mpoint/.style={inner sep = 1.7pt, circle,draw,fill=black},
ypoint/.style={inner sep = 3pt, circle,draw,fill=yellow},
xpoint/.style={inner sep = 3pt, circle,draw,fill=red},
FIT/.style args = {#1}{rounded rectangle, draw,  fit=#1, rotate fit=45, yscale=0.5},
FITR/.style args = {#1}{rounded rectangle, draw,  fit=#1, rotate fit=-45, yscale=0.5},
FIT1/.style args = {#1}{rounded rectangle, draw,  fit=#1, rotate fit=45, scale=2},
vecArrow/.style={
		thick, decoration={markings,mark=at position
		   1 with {\arrow[thick]{open triangle 60}}},
		   double distance=1.4pt, shorten >= 5.5pt,
		   preaction = {decorate},
		   postaction = {draw,line width=0.4pt, white,shorten >= 4.5pt}
	}
]

\begin{scope}[yscale=1.2, xscale = 1.0]
	\begin{scope}[yscale = 0.7]
		\node(a) at (-4,0) {$(1)$};
		\node(b) at (0,4) {$(12)$};
		\node(c) at (4,0) {$(123)$};
		\node[mpoint](ac) at (0,0) {};
		\node[mpoint](ab) at (-2,2) {};
		\node[mpoint](bc) at (2,2) {};
		\node[mpoint](abc) at (0,1.5) {};
	
		\draw [dotted] (a)--(ab)--(b)--(bc)--(c)--(ac)--(a);
		\draw (ab)--(abc) (bc)--(abc);
		\draw [dotted] (b)--(abc)  (c)--(abc);
		
		\draw [dotted] (a)--(abc) node[midway,sloped,above] {\scriptsize$(1), (1, 12, 123)$};
		\draw (ac)--(abc)  node[midway,right] {\tiny$\substack{(1, 123), \\(1, 12, 123)}$};
		\node at (-1, 0.5) {\tiny$\substack{(1), (1, 123), \\(1, 12, 123)}$};
		
		\node [below] at (ac) {$(1, 123)$};
		\node [left] at (ab) {$(1, 12)$};
		\node [right] at (bc) {$(12, 123)$};
		\node [above] at (abc) {$(1, 12, 123)$};

		\node [below] at (-2, 0) {\scriptsize$(1), (1, 123)$};
	\end{scope}
\end{scope}

\end{tikzpicture}	
	\caption{A simplex in the second barycentric subdivision $\Q^d$. An illustration of the barycentric subdivision of a $2$-dimensional simplex in $\B^d$. The simplex is spanned by the chain $1 \subsetneq 12 \subsetneq 123$.
 The single $2$-dimensional simplex is partitioned to six new $2$-dimensional simplexes. 
 The new simplexes are chains of chains.}
	\label{fig-DOUBLE-BARYCENTRIC}       
\end{figure}

The coloring of $\B^d$ defined by the closed sets $F_T$ and $A_+$ 
is simplicial with respect to $\Q^d$; in particular, 
\begin{center}
{\em the simplex $\sigma = \{v_1 \subsetneq \ldots \subsetneq v_m\}$
is contained in $F_T$ if and only if $T \in v_1$.}
\end{center}
Indeed, express $x$ in 
$\sigma = \{v_1 \subsetneq \ldots \subsetneq v_m\}$
as
$$x = \sum_j \beta_j v_j =  \sum_j \beta_j \sum_q \frac{c_j}{w(T_{j,q})} T_{j,q}
= \sum_i \frac{1}{w(T_i)} \Big(\sum_{j : T_i \in v_j} c_j \beta_j\Big)  T_i .$$ 
This means that
$w(T) \lambda_T(x) = h(x)$ iff $T \in v_1$
because $v_1 \subsetneq v_j$ and $c_j,\beta_j >0$
for all $j$.

We now change some of the closed sets $F_T$ to open sets.
Let, as before, $t = \lceil (d+3)/2 \rceil$, and $s = \lceil t /2\rceil$.
Keep all sets $F_T$ for $|T| > s$ without change.
Change all sets $F_T$ for $|T| \leq s$ to be open as follows.
A vertex $v = \{T_1 \subsetneq \ldots \subsetneq T_r\}$ is called {\em low}
if $|T_r| \leq s$.
A simplex $\sigma = \{v_1 \subsetneq \ldots \subsetneq v_m\}$
is {\em low} if $v_1$ is low.
Let $F'_T$ be the set obtained from $F_T$ by
first removing all simplices that are not low,
and then adding the interiors of all low simplices $\sigma$
that contain a low vertex $v \in \sigma$ so that $T \in v$.

We claim that the new sets cover the sphere,
and that the cover number is at most $s$.
Let $\sigma = \{v_1 \subsetneq \ldots \subsetneq v_m\}$ be the support of the point $x$.
There are two cases to consider:

\begin{enumerate}

\item If $\sigma$ is not low, then $x$ is in one of the closed sets $F_T$
for $T \in v_1$, and $x$ does not belong to any open set.
The sets $F_T$ that cover $x$ are in the chain $v_1$
and have size $s < |T| < t$. There is also potentially the set $A_+$ that contains $x$.
The cover number of $x$ is at most $t - s \leq s$.

\item If $\sigma$ is low, then $x$ is covered only by open sets.
Let $v_r$ be the maximal low vertex in $\sigma$.
The point $x$ is in the interior of $\sigma$.
The point $x$ is thus colored by $F'_T$ for $T \in v_r$.
The number of such $T$'s is at most $s$ because $v_r$ is low.
\end{enumerate}

We now claim that the sets $F'_T$ for $|T|\leq s$ are open.
If $x$ is in $F'_T$ then its support $\sigma = \{v_1 \subsetneq \ldots \subsetneq v_m\}$
is low and $T \in v_j$ for some $j$.
Because $x$ is in the interior of $\sigma$,
if $U$ is a small open neighborhood of $x$
then for every $\sigma'$ so that $U \cap \sigma' \neq \emptyset$,
we have $\sigma' \supseteq \sigma$.
This means that $\sigma'$ is low
and contains $v_j$ so that the interior of $\sigma'$ is also in $F'_T$.

It remains to verify that the set $F'_T$ is antipodal-free.
If $x$ is in $F'_T$ then its support $\sigma = \{v_1 \subsetneq \ldots \subsetneq v_m\}$
is low and $T \in v_j$ for some $j$.
Write $x$ as 
$$x = \sum_i \frac{1}{w(T_i)} \Big(\sum_{j : T_i \in v_j} c_j \beta_j\Big)  T_i.$$ 
It follows that $T$ is in the support of $x$ and
the set $\nu(T)$ is incomparable to $T$.
Because $v$ is a chain, we can conclude that $\nu(x) \not \in F_T$.

\end{proof}

\section{Proof of \Cref{t:CC}}

For accessibility, we restate the theorem that we are going to prove.

\begin{customthm}{B}[Colorful chains]
Given a positive integer $n$ and a Kneser coloring of $\P([n])$, there exists a chain of subsets that receives at least $\lfloor n/2 \rfloor + 1$ distinct colors. This bound is sharp; there are Kneser colorings that assign no more than $\lfloor n/2 \rfloor + 1$ colors to any chain.
\end{customthm}
\begin{proof}
As was noted in the introduction, the sharpness of the bound is witnessed by coloring every set $A \subseteq [n]$ of size at most~$n/2$ into a color $i\in A$, and coloring all subsets of size greater than $n/2$ into a distinct color~$c$. We now proceed with the proof of the lower bound.

Let the barycentric subdivision $\B^{n-2}\cong \S^{n-2}$ and the involution $\nu$ on it be as per \Cref{def-barycentric}. Recall that the vertices of $\B^{n-2}$ are precisely nontrivial subsets of $[n]$. Note that each point $x\in \B^{n-2}$ belongs to an interior of exactly one simplex $s = S_1, \dots, S_k$ and has in it \emph{barycentric coordinates} $\lambda_1, \dots, \lambda_k$; that is, $x = \sum_{i=1, \dots, k} \lambda_i S_i$, where $0<\lambda_i$, for $i=1, \dots, k$, and $\sum_{i=1, \dots, k} \lambda_i = 1$.

%Let us define an \emph{antipodal map} $\nu\colon \B^{n-1} \rightarrow \B^{n-1}$ as 
%$$ \nu(x) = \sum_{i=1, \dots, k} \lambda_i ([n] - S_i),$$
%where $x = \sum_{i=1, \dots, k} \lambda_i S_i$ is the barycentric representation of $x$. It is easy to see that, thus defined, $\nu$ is a continuous fixed-point free involution of $\B^{n-1}$. Indeed, the sequence $([n] - S_i)$, for $i=1, \dots, k$, is a chain, and so $\nu$ is well-defined. By construction, it is involutive. Finally, $S_1$ is incomparable to $([n] - S_1)$, and so it cannot be a member of the sequence $([n] - S_i)$, for $i=1, \dots, k$. Thus, $x\neq \nu(x)$ and so $\nu$ is fixed-point free.

Given a Kneser coloring $\kappa$ of $\P([n])$, let us denote by $d(\kappa)$ the maximal number of distinct colors of a chain under $\kappa$. Let $C$ be the set of colors of $\kappa$, and, for $c\in C$, let $e_c = e(c)$ be the unit vector of the coordinate $c$ in $\R^C$. That is, $e_c(c) = 1$ and $e_c(c') = 0$, for all $c'\in C - c$. Let us define a map $\varphi_\kappa \colon \B^{n-2}\rightarrow \R^C$ as

$$\varphi_\kappa(x) = \sum_{i=1, \dots, k} \lambda_i e_{\kappa(S_i)},$$
where, again, $x = \sum_{i=1, \dots, k} \lambda_i S_i$ is the barycentric representation of $x$. Then $\varphi_\kappa$ is a continuous, moreover, piecewise-linear map. Let $\PP = \varphi_\kappa(\B^{n-1})$. Then $\PP\subseteq \R^C$ is a simplicial complex whose simplicies are images of the simplices of $\B^{n-2}$ under $\varphi_\kappa$. By construction, for any simplex $\sigma$, the dimension of $\varphi_\kappa(\sigma)$ is $\left|\kappa(\sigma)\right| - 1$. So, the dimension of $\PP$ is at most $d(\kappa) - 1$. Note that ``at most'' comes from the fact that the dimension of $\PP$ is only affected by the chains that do not contain $[n]$, while $d(\kappa)$ is the maximum over all chains.

We claim that $\varphi_\kappa$ does not collapse antipodal points. Indeed, suppose for some $x = \sum_{i=1, \dots, k} \lambda_i S_i$, $\varphi_\kappa(x) = \varphi_\kappa(\nu(x))$. That is

\begin{align*}
\sum_{i=1, \dots, k} \lambda_i  e_{\kappa(S_i)} &= \varphi_\kappa(x) = \varphi_\kappa(\nu(x)) \\
	&= \varphi_\kappa\left(\sum_{i=1, \dots, k} \lambda_i S_i\right) 
	= \sum_{i=1, \dots, k} \lambda_i  e_{\kappa([n] - S_i)}.
\end{align*}

Note that $S_1$ is the smallest subset in $S_1, \dots, S_k$. Then, for any $1\leq j\leq k$, $[n] - S_1 \supseteq [n] - S_j$, and hence $S_1 \cap ([n] - S_j) = \emptyset$, and so, as $\kappa$ is Kneser, $\kappa(S_1) \neq \kappa([n] - S_j)$. Then the projection of the left-hand side on $e_{\kappa(S_1)}$ is $\lambda_1 \neq 0$, and the similar projection of the right-hand side is $0$, a contradiction. 

But then $\varphi_\kappa$ is a continuous function from $\S^{n-2}$ into $\PP$, which does not collapse antipodal points, and so, by \Cref{thm:Jaw}, $2d(\PP)\geq n-2+1$. As $d(\PP) \leq d(\kappa)-1$, we get $2d(\kappa) - 2 \geq n-1$ and so $d(\kappa)\geq (n+1)/2$. As $d(\kappa)$ is integer, it improves to $d(\kappa)\geq \lceil (n+1)/2\rceil = 1+\lfloor n/2\rfloor$. 
\end{proof}

\section{Proof of Theorem~\ref{t:repagn}}

\begin{customthm}{C}[Agnostic replicability and global stability]
The following statements are equivalent for a concept class $\H$:
\begin{itemize}
\item $\H$ is finite.
\item $\rho_{\mathtt{agn}}(\H)>0$ (equivalently, $\List_{\mathtt{agn}}(\H)<\infty$). 
%that is, $\H$ is agnostically globally stable learnable (equivalently, agnostically list-replicable learnable).
\end{itemize}
\end{customthm}
\begin{proof}

The direction that every finite class $\H$ is agnostically learnable by a list-replicable learning rule
follows by considering any proper learning rule. Indeed, such a rule always outputs a classifier in $\H$
and thus, 
\[\List_{\mathtt{agn}}(\H)\leq \lvert\H\rvert < \infty.\]

The converse direction is more involved;
% it suffices to show that for every $L>0$ there exists $\varepsilon^\star>0$ for which there exists no \((L,\varepsilon^\star)\)-list replicable learner for $\H$.
%\subsection*{Picking $\varepsilon^\star=\eps^\star(\rho)$}
We rely on the following definition from~\cite{AlonGHM23}.
\begin{definition}[$\gamma$-realizability and interpolation~\citep*{AlonGHM23}] \label{def:gammaRealizable}
Let $\H\subseteq\{\pm 1\}^\X$ be a concept class, let $\gamma\in(0,1)$. 
    A sequence~$S=((x_1,y_1),\ldots,(x_m,y_m))$ is \emph{$\gamma$-realizable} with respect to $\H$
    if for any probability distribution $Q$ over $S$ there exists $h\in\H$ such that 
\[
 \mathsf{corr}_Q(h) := 1-2L_Q(h) =  \E_{(x,y)\sim Q}[h(x)\cdot y] \ge \gamma.
\]
We say that a set $\{x_{1},\ldots,x_{d}\}\subseteq\mathcal{X}$ is $\gamma$-interpolated by $\H$
if for any $y_1,\ldots,y_d\in\{\pm 1\}$, the sequence $S=(x_{1},y_1)),\ldots,(x_{d},y_{d})$ 
is $\gamma$-realizable with respect to $\H$.
\end{definition}
Note that  $\gamma$-interpolation specializes the concept of shattering in the context of VC theory.  

Let $\H$ be an infinite concept class; it will be convenient to assume that $\H$ is symmetric in the sense that $h\in \H$ if and only if $-h\in \H$ for every $h\in\H$, where $(-h)(x)=-h(x)$ for all $x\in \X$. This assumption does not compromise generality:
indeed, notice that $\H\cup (-\H)$ is symmetric, where $-\H=\{-h : h\in \H\}$ and that
\begin{align*}
\List_{\mathtt{agn}}(\H\cup(-\H)) &\leq \List_{\mathtt{agn}}(\H) + \List_{\mathtt{agn}}(-\H)\\
                                  &= 2\List_{\mathtt{agn}}(\H).
\end{align*}
Thus, $\List_{\mathtt{agn}}(\H\cup(-\H))<\infty$ if and only if $\List_{\mathtt{agn}}(\H)<\infty$.
Hence, if $\H$ is not symmetric we replace it with~$\H\cup(-\H)$, which is symmetric.

\begin{lem}\label{lem:interpolation}
    Let $\H$ be an infinite symmetric concept class. Then, for every $d\in\mathbb{N}$ there is $\gamma>0$
    and a set $\{x_{1},\ldots,x_{d}\}\subseteq\mathcal{X}$ which is $\gamma$-interpolated by $\H$.
\end{lem}
Before proving Lemma~\ref{lem:interpolation}, we show how to use it to derive Theorem~\ref{t:repagn}.
Assume towards contradiction that $\H$ is agnostically list-replicably learnable with list size $\List_{\mathtt{agn}}(\H)=K<\infty$.
By Lemma~\ref{lem:interpolation}, there exists $\gamma>0$ and a set
\(U=\{x_{1},\ldots,x_{2K}\}\subseteq\mathcal{X}\) which is $\gamma$-interpolated by $\H$.
Consider the set~$\Theta$ of all distributions supported on sequences $(x_1,y_1),\ldots (x_{2K},y_{2K})$,
where $y_i$'s are in $\{\pm 1\}$.
Notice that $\Theta$ is isomorphic to the $2K-1$-dimensional $\ell_1$-sphere (i.e.\ the boundary of the $2K$-dimensional $\ell_1$-ball). 
This follows by identifying each vector $(t_1,\ldots, t_{2K})$ on the $\ell_1$-sphere (i.e.\ $\sum\lvert t_i\rvert = 1$) 
with the distribution that assigns probability $\lvert t_i\rvert$ to the example $(x_i,\mathtt{sign}(t_i))$.
This transformation maps the standard involution~``$x\to -x$'' in~$\R^{2K}$ to a fixed-point free involution $\nu:\Theta\to \Theta$:
for every distribution $p$ supported on $\{(x_i,y_i)\}_{i=1}^{2L}$, the distribution $q=\nu(p)$ is supported
on $\{(x_i,-y_i)\}_{i=1}^{2K}$ and $q(x_i,-y_i):=p(x_i,y_i)$.
Notice that for every distribution $\D\in\Theta$ and for every hypothesis~$h$,
\begin{equation}\label{eq:1}
L_{\D}(h) + L_{\nu(\D)}(h) = 1.
\end{equation}
Further, notice that each distribution $\D\in\Theta$ is supported on a $\gamma$-realizable sequence
$\{(x_i,y_i)\}_{i=1}^{2K}$. In particular, 
\begin{equation}\label{eq:2}
(\forall \D\in \Theta)(\exists h\in \H): L_{\D}(h)\leq \frac{1-\gamma}{2}.
\end{equation}
%$\nu(\D_{x_i,y_i}) = \D_{x_i, -y_i}$ and on general distributions in $B$ it is defined by linear extension:
% \begin{align*}
% \nu(\D)=\nu\Bigl(\sum_{i=1}^{2L}\alpha_i\D_{x_i, +1} + \beta_i\D_{x_i, -1}\Bigr) = \sum_{i=1}^{2L}\beta_i\D_{x_i, +1} + \alpha_i\D_{x_i, -1}. 
% %\tag{$\alpha_i,\beta_i\geq 0$ and $\sum_{i}(\alpha_i+\beta_i)=1.$}
% \end{align*}

Now, let $\A$ be an agnostic $(\frac{\gamma}{2}, K)$-list replicable learner.
%; we will prove that $K\geq L$, which is a contradiction to the definition of $L$, as $L=K+1$. 
Pick the confidence parameter $\delta>0$ to be sufficiently small 
such that for every distribution $\D\in \Theta$ there exists a hypothesis $h$ satisfying:
\begin{itemize}
\item Given $n(\frac{\gamma}{2}, K, \delta)$ i.i.d examples drawn from $\D$,
the learner $\A$ outputs $h$ with probability $>\frac{1}{K+1}$.
\item $L_\D(h)< \frac{1-\gamma}{2}+\frac{\gamma}{2}=\frac{1}{2}$. (This follows by Equation~\ref{eq:2}
because $\A$ is an agnostic $(K,\frac{\gamma}{2})$-list replicable learner).
\end{itemize}
Let $H_\D=\{h\vert_U: \text{$h$ satisfies the above items}\}$. Thus, $0< \lvert H_\D\rvert \leq K$. For every $h:U\to\{\pm 1\}$, define
\[C_h = \{\D\in\Theta: h\in H_\D \}.\] 
Notice that each $C_h$ is open. We claim that $\{C_h : h:U\to\{\pm 1\}\}$ is an antipodal-free cover of $\Theta$.
Indeed, it covers $\Theta$ because $H_\D\neq\emptyset$ for every distribution $\D\in\Theta$.
That the cover is antipodal free follows by Equation~\ref{eq:1}, 
because every $C_h$ consists only of distributions $\D$ for which $L_\D(h)<\frac{1}{2}$.
By Theorem~\ref{t:localLS} we get that there exists a distribution $\D\in \Theta$ such that $\D$ belongs to at least $\lceil(2K-1 + 3)/2 \rceil = K+1$ distinct sets $C_{h_i}$;
consequently, each of these $h_i$ belongs to $H_\D$, which implies that $\lvert H_\D\rvert\geq K + 1$, yielding the desired contradiction.
\end{proof}

\subsection*{Proof of Lemma~\ref{lem:interpolation}}
 Lemma~\ref{lem:interpolation} is a direct corollary of the next two lemmas.
\begin{lem}\label{lem:ind}
Let $\H\subseteq\{\pm 1\}^\X$ be an infinite class.
Then, for every $d\in \mathbb{N}$ there exist $d$ hypotheses $h_1\ldots h_d\in \H$
and $d$ points $x_1,\ldots x_d \in \X$ such that the $d$ vectors $(h_i(x_1),\ldots h_i(x_d))$, $i=1,\ldots,d$
are linearly independent over~$\mathbb{R}$.
\end{lem}

\begin{lem}\label{lem:indint}
Assume that $h_1,\ldots,h_d$ and $x_1,\ldots,x_d$ are as in the conclusion of Lemma~\ref{lem:ind}.
Then, there exists $\gamma>0$ such that the class $\{\pm h_i : i=1,\ldots,d\}$ $\gamma$-interpolates
$\{x_1,\ldots x_d\}$.
\end{lem}
\begin{proof}[Proof of Lemma~\ref{lem:ind}]
Assume towards contradiction that for some $d$ there are no $h_i$'s and $x_i$'s as stated.
Then, by basic linear algebra, there exists a basis  $\{h_1,\ldots, h_k\}$, for $k<d$, of the linear span of $\H$ over\footnote{Here we treat $\{\pm 1\}^\X$ as a subset of the linear space $\mathbb{R}^\X$.} $\mathbb{R}$.
Now, pick~$x_1,\ldots x_k\in \X$ such that the $k$ vectors $(h_i(x_1),\ldots, h_i(x_k))$ are linearly independent.
It follows that any function $f:\X\to\R$ in the linear span of $\H$ is uniquely determined by its values on $x_1,\ldots x_k$.
(That is, if $f_1,f_2$ are in the linear span of $\H$, and $f_1(x_i)=f_2(x_i)$ for all $i\leq k$ then~$f_1=f_2$).
Thus, the number of hypotheses in $\H$ is at most the number of $\{\pm 1\}$-valued functions on $\{x_1,\ldots, x_k\}$,
which is $2^k$. This contradicts the assumption that $\H$ is infinite.
\end{proof}
\begin{remark}
The proof of Lemma~\ref{lem:ind} yields an analogue of the Sauer-Shelah-Perles (SSP) Lemma~\citep{Sauer72Lemma} where the VC dimension is replaced by the linear dimension. The SSP Lemma asserts that if a set of $n$-bit vectors has VC dimension $d$ then there are at most ${n \choose \leq d}$ vectors in this set. The above argument gives that if a set of binary vectors has linear dimension $d$ then there are at most $2^d$ vectors in the set.
\end{remark}

\begin{proof}[Proof of Lemma~\ref{lem:indint}]
Let $\bar h_i = (h_i(x_1),\ldots, h_i(x_d))\in \{\pm 1\}^d$.
By assumption, \(\{\bar h_i : i\leq d\}\) is a basis of $\R^d$. Hence for every $\bar y\in\{\pm 1\}^d$
there are coefficients $\alpha_{\bar y,i}$ such that $\sum_{i}\alpha_{\bar y,i}\bar h_i = \bar y$.
Set 
\[\gamma = \min_{\bar y\in\{\pm 1\}^d} \frac{1}{\sum_{i=1}^d \lvert \alpha_{\bar y,i}\rvert} > 0.\]

Fix an arbitrary $\bar y = (y_1,\ldots, y_d)\in \{\pm 1\}^d$. We need to show that the sequence $(x_1,y_1),\ldots,(x_d,y_d)$ is $\gamma$-realizable
by~$\{\pm h_i : i\leq d\}$. In what follows, we denote $\alpha_{\bar y,i}$ by simply $\alpha_i$. Let 
    $$q_i = \frac{\lvert \alpha_i\rvert}{\sum_{i=1}^d \lvert \alpha_i\rvert}$$ 
and $\bar v_i = \sign(\alpha_i) \bar h_i \in \{\pm \bar h_i : i\leq d\}$.
Note that $q_i\geq 0$ for all $i$
and that $\sum_{i}q_i=1$. We thus can treat $q_i$'s as a probability distribution $q$ over the class $\{\pm h_i: i\leq d\}$. Then
\begin{align*} 
    \E_{h\sim q} \bar h = \sum_{i=1}^d q_i \bar v_i &= 
    \frac{\sum_{i=1}^d \bar v_i \lvert \alpha_i\rvert}{\sum_{i=1}^d \lvert \alpha_i\rvert}
    = \frac{\sum_{i=1}^d \bar h_i \alpha_i}{\sum_{i=1}^d \lvert \alpha_i\rvert}
    = \frac{\bar y}{\sum_{i=1}^d \lvert \alpha_i\rvert}.
\end{align*}
Thus, for every $i\leq d$,
\begin{equation}\label{eq:3}
    \E_{h\sim q}[y_i\cdot h(x_i)] 
    = \frac{y_i^2}{\sum_{i=1}^d \lvert \alpha_i\rvert}
    = \frac{1}{\sum_{i=1}^d \lvert \alpha_i\rvert}
    \geq \gamma.
\end{equation}
Now, let $p$ be a probability distribution over the sample $(x_1,y_1),\ldots, (x_d,y_d)$. Then,
\begin{align*}
\E_{h\sim q} \E_{(x,y)\sim p}[y\cdot h(x)] 
&= \E_{(x,y)\sim p} \E_{h\sim q}[h(x)\cdot y]\\
    &\geq \E_{(x,y)\sim p} [\gamma] \tag{By Equation~\ref{eq:3}}\\
    &=\gamma.
\end{align*}
Thus, in particular, there exists $h\in\{\pm h_i : i\leq d\}$ for which \(\E_{(x,y)\sim p}[h(x)\cdot y]\geq \gamma\), as required.

\end{proof}

\section{Proofs of Theorems \ref{t:weaklower} and \ref{t:finite}}

We will utilize some machinery from Section~3.3 in~\cite{Chase23rep}.

Let $\H$ be a finite class. Denote by $\Delta = \Delta_\H$ the collection of $\H$-realizable distributions endowed with the \emph{total-variation distance} ($\TV$). Recall that, by the definition, for distributions $\D_1$ and $\D_2$ over $X\times\{\pm1\}$:

\begin{align*}
	\TV(\D_1, \D_2) &= \sum_{e \colon \D_1(e)\geq \D_2(e)}  (\D_1(e) - \D_2(e))
	= \frac{1}{2}\sum_{e}  \left|\D_1(e) - \D_2(e)\right|.
\end{align*}
We find it helpful to think of $\Delta$ as being isometrically embedded into $[0,1]^{X\times\{\pm1\}}$ with half of the $l_1$ norm of a difference as a metric. In particular, $\Delta$ is a closed subspace of a compact metric space, and hence is compact.

\begin{theorem}[Theorem 7 in \cite{Chase23rep}]\label{t:old-t7}
The following statements are equivalent for a finite class $\H$ and $\varepsilon > 0$:
\begin{enumerate}
	\item $\List(\H, \varepsilon) \leq L$;
	\item There exists $\delta > 0$ and a coloring $\D\mapsto h_\D$ of $\Delta(\H)$, where $h_\D \in \{\pm1\}^X$ is an hypothesis, such that $L_\D(h_\D)\leq \varepsilon$ and $\left|\left\{h_{\D'}~:~\TV(\D, \D')\leq \delta \right\}\right| \leq L$, for every $\D\in \Delta$.
\end{enumerate}
\end{theorem}
In the original statement of Theorem 7 in \cite{Chase23rep}, condition (1) was stated as ``$\List(\H) \leq L$'', and (2) as ``For every $\varepsilon > 0$ $\dots$''. Our \Cref{t:old-t7} is thus a slight strengthening of it, and the original theorem follows from it by letting $\varepsilon\rightarrow 0$. The original proof, however, essentially proves the strengthened statement.
%We give a brief overview of the proof, the full proof itself is not much harder.
%\begin{proof}
%In the direction (1$\Rightarrow$2), given a $(\varepsilon,L)$-list replicable learner $\A$, let us pick $n$ so that
%\[\Pr_{S\sim \D^n}[\A(S)\in \L] \geq 1- \frac{1}{3L}, \]
%where $\L = \L(\D)$ is the list guaranteed by list replicability (so that $\lvert \L\rvert = L$ and $L_\D(h)\leq \varepsilon$ for all $h\in\L$). We then color a distribution $\D\in \Delta$ into a hypothesis
%that is most frequently outputted by $\A$ when applied on $S\sim \D^n$. That is,
%\begin{align*}
%h_{\D} = \mathop{\argmax}_h \Pr_{S\sim \D^n}[\A(S) = h], 
%\end{align*}
%where ties are broken arbitrarily. It is then easy to see that there is a small $\delta$ such that all colors used in a $\delta$-ball around every $\D\in \Delta$ will only have colors from $\L(\D)$, that guarantees that both conditions on the coloring in (2) hold.
%
%In the other direction, we use the finiteness of $\H$. Then, for a sufficiently large sample size $n$, the empirical distribution $\hat{\D}$, induced by a sample $S\sim \D^n$, is w.h.p. close to $\D$ in total-variation. Thus, the learning rule that outputs the hypothesis $h=h_{\hat{\D}}$ is a list replicable learner for $\H$.
%\end{proof}

\begin{cor}\label{c:old-t7-closed}
	For a finite class $\H$, $\List(\H, \varepsilon)$ is the minimal integer $L$ such that there exist a closed cover $\F = \left\{A_h~|~h\in \{\pm1\}^X\right\}$ of $\Delta$ of overlap-degree at most $L$, such that $L_\D(h) \leq \varepsilon$ for all $h\in \{\pm1\}^X$ and $\D\in A_h$.
\end{cor}
\begin{proof}
	The proof is by establishing the equivalence of the statement with \Cref{t:old-t7} (2).
	
	In one direction, suppose that, for a given $\varepsilon >0$, we have $\delta$ and a coloring $\zeta\colon \Delta \rightarrow \{\pm1\}^X$ that satisfy the conditions of \Cref{t:old-t7} (2). For an hypothesis $h$, let us define $A_h$ as the closure of $\zeta^{-1}(h)$. By construction, it is a closed cover of $\Delta$. Let us take arbitrary $\D\in A_h$, then $\D$ can be approximated by the sequence $\D_i$ of realizable distributions such that $\zeta(\D_i) = h$, and hence $L_{\D_i}(h)\leq \varepsilon$. As the population loss is continuous with respect to the distribution, $L_{\D}(h)\leq \varepsilon$, as needed. Finally, if an arbitrary $\D$ is covered by $A_h$, then there is $\D'$ in the $\delta$-ball around $\D$ for which $\zeta(\D') = h$, and so there are at most $L$ such sets.
	
	In the other direction, again, for a given $\varepsilon >0$, suppose there is a closed cover $\F$ from the statement of the corollary. Let us take $\zeta$ to be an arbitrary coloring consistent with $\F$, that is, such that for every $\D\in \Delta$, $\D \in A_{\zeta(\D)}$. Trivially, $L_\D(\zeta(\D))\leq \varepsilon$, for any $\D\in \Delta$, and we only need to find suitable $\delta$ to satisfy the second condition on the coloring.
	
	Let $\G$ be a closed cover obtained from $\F$ by $\infty$-normality property, see \Cref{sec-normal-spaces} for the definition. That is, $\G = \left\{B_h~|~h\in \{\pm1\}^X\right\}$, $A_h \subseteq B^\circ_h$ for all $h\in \{\pm1\}^X$, and for any $H\subseteq \{\pm1\}^X$, $\bigcap_{h\in H}B_h$ is nonempty if and only if $\bigcap_{h\in H}A_h$ is. In particular, $\G$ and $\F$ have the same overlap-degree, which is at most $L$. By compactness of $\Delta$, $\TV(A_h, \Delta - B^\circ_h) >0$ for every $h\in \{\pm1\}^X$. Let $\delta = \frac{1}{2}\min_{h} \TV(A_h, \Delta - B^\circ_h)$, and let us take an arbitrary $\D\in \Delta$. By the choice of $\delta$, if $A_h$ intersects the $\delta$-ball $U_\delta(\D)$ around $\D$, then $\D\in B_h$. So 
	\begin{align*}
		\left|\left\{h_{\D'}~:~\TV(\D, \D')\leq \delta \right\}\right|
				&\leq \left|\left\{h~:~A_h \cap U_\delta(\D)\neq \emptyset \right\}\right| \\
				&\leq \left|\left\{h~:~\D \in B_h \right\}\right| \leq L,
	\end{align*}
	where the last inequality is by the bound on the overlap-degree of $\G$.
\end{proof}

It will be convenient to have an alternative to $\Delta$ geometric representation of the realizable distributions of~$\H$. For $\D\in\Delta$, let $\D_\theta\in [-1, 1]^X$ be defined as $\D_\theta(x) = \D(x, 1)$ if $\D(x, 1) >0$ and $\D_\theta(x) = -\D(x, -1)$ otherwise. Let $\Theta = \{\D_\theta~:~\D\in \Delta\}$. Note that, as $\D$ is realizable, at most one of $\D(x, 1)$ and $\D(x, -1)$ is nonzero, for every $x\in X$. From this easily follows that the map $\theta\colon \D\mapsto \D_\theta$ is a bijection, with the inverse of $\theta$ given by $\D(x, 1) = \max(0, \D_\theta(x))$ and $\D(x, -1) = \max(0, - \D_\theta(x))$, for all $x\in X$. Moreover, it is easy to notice that if we endow $\Theta \subseteq [-1, 1]^X$ with half of the $l_1$ norm of a difference as a metric, just as we did for $\Delta$, then $\theta$ becomes an isometry between $\Delta$ and $\Theta$. Because of that, both \Cref{t:old-t7} and \Cref{c:old-t7-closed} remain valid if we use $\Theta$ instead of $\Delta$.

\begin{customthm}{D}[Lower bound for list replicable weak learners]
Let $\H$ be a concept class. Then, for any $0 < \varepsilon < 1/2$, it holds 
$\List(\H, \varepsilon)\geq \max\{ 1 + \lceil \vc(\H)/2\rceil, 1 + \lfloor \vc^\star(\H)/2\rfloor\}$.
\end{customthm}
\begin{proof}
	The statement effectively consists of the two bounds: $L\geq 1 + \lfloor \vc(\H)/2\rfloor$, and  $L\geq 1 + \lfloor \vc^*(\H)/2\rfloor$. We prove them separately, however, the proof strategy is similar: We identify the part of $\H$ that witnesses the corresponding dimension, find a subcomplex isomorphic to a sphere in $\Delta$, argue that the coloring induced by the weak learner via \Cref{c:old-t7-closed} is antipodal-free, and apply the local LS bound (\Cref{t:localLS}) to it. Thus, we give a detailed proof of the $\vc^\star$ lower bound, and then, for the proof of the $\vc$ lower bound we gloss over the common part, and concentrate on the differences between the cases. Throughout the proof, we assume that $Y = \{0, 1\}$, instead of $Y=\{\pm1\}$ used elsewhere.

	\emph{($L\geq 1 + \lfloor \vc^*(\H)/2\rfloor$).} As $\List(\H, \varepsilon)$ is increasing in $\H$, without losing generality we can assume that $\H$ is the minimal class witnessing $\vc^\star(\H) = m$. That is, $\H$ is the class $\H_m$ from \Cref{t:finite}.
	
	We will identify the domain $X$ with $\P([m])$.	Again, let $\Delta = \Delta(\H)$. For $A\subseteq [m]$ such that $A\neq \emptyset, [m]$, let us define a distribution $\D_A$ on $X \times {\{0,1\}}$ that assigns probabilities $1/2$ to examples $(A, 1)$ and $([m] - A, 0)$. Note that, for any $i\in [m]$, $L(\D_A, h_i)=0$ if $i\in A$ and $L(\D_A, h_i)=1$ if $i\in [m] - A$; in particular, $\D_A$ is realizable by $\H_m$. Let us take a chain of sets $\emptyset \subsetneq A_1 \subsetneq \dots \subsetneq A_k \subsetneq [m]$. For any  $\lambda_1, \dots, \lambda_k$ such that $0<\lambda_i$ and $\sum_{i=1, \dots, k} \lambda_i = 1$, let 
	$\D = \sum_{i=1, \dots, k} \lambda_i \D_{A_i}$. It is easy to see that thus constructed $\D$ is also a distribution realizable by $\H_m$, with $L(\D, h_i)=0$ for any $i\in A_1$ and $L(\D, h_i)=1$ for any $i\in [m] - A_k$; in particular, $\D$ is realizable.
	
	Now, let $\Delta_R\subseteq \Delta$ be a simlicial complex with vertices $\D_A$, for all $A\subseteq [m]$ such that $A\neq \emptyset, [m]$, and with simplices spanned by $\D_{A_1}, \dots, \D_{A_k}$, for all chains $\emptyset\subsetneq A_1\subsetneq \dots \subsetneq A_k\subsetneq [m]$. As argued above, all distributions in $\Delta_R$ are realizable, that is, $\Delta_R$ is indeed a subspace of $\Delta$. Moreover, any two simplices of $\Delta_R$ can intersect only on a subsimplex, and so $\Delta_R$ is a valid geometric realization of the similarly constructed abstract simplicial complex. Note that as an abstract simplicial complex, $\Delta_R$ is the barycentric subdivision of the boundary of the $(m-1)$-dimensional simplex. Thus, $\Delta_R$ is a topological sphere $\S^{m-2}$.
	
	The continuous fixed-point free involution $\nu\colon \Delta_R\rightarrow \Delta_R$ is defined as in \Cref{def-barycentric}. Namely, for $\D\in \Delta_R$ with the barycentric representation $\D = \sum_{i=1, \dots, k} \lambda_i \D_{A_i}$,
		$$ \nu(\D) = \sum_{i=1, \dots, k} \lambda_i \D_{[m] - A_i}.$$
		
	We are now going to show that for $\varepsilon < 1/2$, any closed cover $\F = \left\{C_h~|~h\in \{0,1\}^X\right\}$ of $\Delta_R$ that satisfies the population loss condition from \Cref{c:old-t7-closed}, that is, such that $L(\D, h) \leq \varepsilon$ for all $h\in \{0,1\}^X$ and $\D\in C_h$, should have overlap-degree at least $1 + \lfloor m/2\rfloor$. This will be done through an application of \Cref{t:localLS}. By \Cref{c:old-t7-closed}, this would clearly yield the desired lower bound on $\List(\H_m, \varepsilon)$.
	
	So let $\F$ be such closed cover. We argue that every $C_h$ is antipodal-free. Towards a contradiction, suppose that $C_h$ contains $\D=\sum_{i=1, \dots, k} \lambda_i \D_{A_i}$ and $\nu(\D) = \sum_{i=1, \dots, k} \lambda_i \D_{[m] - A_i}$. Then
	\begin{align*}
		L(\D, h) &= L\left(\sum_{i=1, \dots, k} \lambda_i \D_{A_i}, h\right) 
		= \sum_{i=1, \dots, k} \lambda_i L(\D_{A_i}, h) \\
		&= \sum_{i=1, \dots, k} \lambda_i \left(\frac{1}{2}(1-h(A_i)) + \frac{1}{2}h\left([m] - A_i\right) \right) \\
		&= \frac{1}{2} + \frac{1}{2}\sum_{i=1, \dots, k} h\left([m] - A_i\right) - h(A_i).
	\end{align*}
	Similarly,
	$$L(\nu(\D), h) = \frac{1}{2} - \frac{1}{2}\sum_{i=1, \dots, k} h\left([m] - A_i\right) - h(A_i).$$
	Thus, $L(\D, h) + L(\nu(\D), h) = 1$, and so at most one of $\D$ and $\nu(\D)$ has loss less than $1/2$, and thus can belong to $C_h$, for $\varepsilon < 1/2$. After that, by \Cref{t:localLS}, the overlap-degree of $\F$ is at least $\lceil (m-2+3)/2\rceil = 1 + \lfloor m/2\rfloor$, as needed.
	
	\emph{($L\geq 1 + \lfloor \vc(\H)/2\rfloor$).}  As before, without losing generality, we assume that $\H$ is the minimal class witnessing $\vc(\H) = m$. That is, $X = [m]$ and $\H$ is the class of the characteristic functions of all subsets of $[m]$.
	
	Let $\Theta = \Theta(\H)$. Then 
		$$ \Theta = \left\{\D_\theta\in [-1, 1]^{X}~:~ \|\D_\theta\|_1=1\right\},$$
	where $\|\cdot\|_1$ is an $l_1$ norm, that is, $\|\D_\theta\|_1 = \sum_{x\in X} |\D_\theta(x)| = \sum_{x, i\in X\times \{0,1\}} \D(x, i)$. In particular, $\Theta$ is an $(m-1)$-dimensional $l_1$-ball, and thus is a topological sphere $\S^{m-1}$. Let us define the continuous fixed-point free involution $\nu\colon \Theta\rightarrow \Theta$ as $\nu(\D_\theta) = -\D_\theta$. Then it is easy to check that we again have the same property on the population loss as before:  $L(\D_\theta, h) + L(\nu(\D_\theta), h) = 1$, for every $h\in \H$.
	
	Then, for $\varepsilon < 1/2$, any closed cover $\F = \left\{C_h~|~h\in \{0,1\}^X \right\}$ of $\Theta$ that satisfies $L(\D, h) \leq \varepsilon$, for all $h\in \{0,1\}^X$ and $\D\in C_h$, is antipodal-free. Hence, by  \Cref{t:localLS}, the overlap-degree of $\F$ is at least $\lceil (m-1+3)/2\rceil = 1 + \lceil m/2\rceil$. By \Cref{c:old-t7-closed}, this implies the same lower bound on $\List(\H, \varepsilon)$, that is, $\List(\H, \varepsilon)\geq 1 + \lceil m/2\rceil$.
\end{proof}

\begin{customthm}{E}[Finite classes]
For every class $\H$  of size $m$ we have
$\List(\H)\leq 1 + \lfloor m/2\rfloor$. 
This is sharp, as witnessed by the class $\H_m=\{h_i : i\leq m\}$ 
of $m$ projection functions on $X=\{0,1\}^m$: $h_i(x)=x_i$ for every $x\in X$ and $i$. For $\H_m$, $\vc^*(\H_m) = m$, and so $\List(\H_m) = \List(\H_m, \varepsilon)= 1 + \lfloor m/2\rfloor$, for any $0<\varepsilon<1/2$. 
\end{customthm}
It is possible, and would be easier, to prove the upper bound on the list replicability number of finite classes via \Cref{c:old-t7-closed}, that is, by presenting a suitable coloring of $\Delta(\H)$. However, to make it more palpable, we do it directly, by presenting the required algorithm.
\begin{proof}
	We only need to prove the $\List(\H)\leq 1 + \lfloor m/2\rfloor$ bound for a class $\H$ of size $m$.
    We assume that the domain $\X$
    is finite: indeed, by identifying domain points that are same-valued by all hypotheses in $\H$ we may assume without loss of generality that $\lvert \X\rvert = M\leq 2^m$. Let $h_\maj\colon X\rightarrow \{0,1\}$ be the \emph{majority vote} of $\H$, that is, an hypothesis such that for any $x\in X$, $\left|\left\{h\in \H~:~ h(x)= h_\maj(x)\right\}\right|\geq |\H|/2$. 
	
	Let us pick arbitrary $\varepsilon>0$ and $\delta>0$ from the definition of a list-replicable learner, and let $n = n(\varepsilon, \delta)$, to be specified later, be big enough. For a realizable by $\H$ distribution $\D$, let $\hat{\D} = \hat{\D}(S)$ be the empirical distribution induced by a sample $S\sim \D^n$. Note that, by construction, $L_S(h) = L_{\hat\D}(h)$, for any $h\in \{0,1\}^X$. Also note that, for any distributions $\D_1$ and $\D_2$ and any $h\in \{0,1\}^X$, $|L_{\D_1}(h) - L_{\D_2}(h)| \leq  \|\D_1 - \D_2\|_1 = 2\TV(\D_1, \D_2)$. By taking $n$ sufficiently big, we can ensure that for any $e >0$, $\hat{\D}$ is $e$-close to $\D$ in total variation with probability at least $1-\delta$. Note that this estimate uses the finiteness of $\H$ and is independent of $\D$.  

    Let us now define a learning rule $\A$ as follows:

    \begin{tcolorbox}
    \begin{center}
    An $(\varepsilon,L)$-list replicable learner $\A$ of $\H$, for $|\H| = m$ and $L\leq 1 + \lfloor m/2\rfloor$.
    \end{center}
    
    \ \ \ \
    {\bf Input:} A sample $S\sim \D^n$, where $\D$ is some distribution, realizable by $\H$.
    
    \medskip
    \begin{enumerate}
    \item If  $L_S(h_\maj)\leq (2M + 2)e$, output $h_\maj$;
    \item Otherwise, output any $h\in \H$, consistent with $S$, with arbitrary probability.
    \end{enumerate}
    \end{tcolorbox}    
	
	If $\D$ is such that the event $[\TV(\D, \hat\D) \leq e]$, for $S\sim \D^n$, is a subset of the event $[L_S(h_\maj)\leq (2M+2)e]$, then, by construction, $\A(S) = h_\maj$ with probability at least $1-\delta$. Also, for arbitrary $\hat\D = \hat\D(S)$, witnessing both events, $|L_\D(h_\maj) - L_S(h_\maj)| \leq 2\TV(\D, \hat\D) \leq 2e$, and so $L_\D(h_\maj) \leq L_S(h_\maj) + 2e \leq (2M+4)e$. So, for such $\D$, the replicable list $\L(\D) $ can be defined as just $\{h_\maj\}$.
	
	Alternatively, suppose that $[\TV(\D, \hat\D) \leq e] \subsetneq [L_S(h_\maj)\leq (2M+2)e]$. That is, there is $\hat\D = \hat\D(S)$ such that $\TV(\D, \hat\D) \leq e$ and $L_S(h_\maj)> (2M+2)e$. Then $L_\D(h_\maj) > 2Me$. In particular, there is $x_0\in X$ such that $\D(x_0, -h_\maj(x_0)) > 2e$. %Note that this implies that for any $\D_1$ such that $\TV(\D, \D_1) \leq e$ it holds $\D_1(x_0, -h_\maj(x_0)) > e$. 
	
	Let us now define $\L = \L(\D)$ as consisting of the following hypotheses:
	\begin{itemize}
		\item $h_\maj$, whenever $L_\D(h_\maj)\leq (2M+4)e$;
		\item those $h\in \H - \{h_\maj\}$ for which $L_\D(h)\leq 2e$.
	\end{itemize}
	Then for any $h\in \L(\D) - \{h_\maj\}$, this implies $h(x_0) = -h_\maj(x_0)$, as otherwise $L_\D(h) \geq \D(x_0, -h_\maj(x_0)) > 2e$.
	So, by the properties of $h_\maj$, $|\L| \leq 1 + \lceil m/2\rceil$. Also, $L_\D(h)$ is linearly bounded with $e$, for any $h\in \L(\D)$. Finally, for any $S\sim \D^n$ such that $\TV(\D, \hat\D) \leq e$, $\A(S)$ is either $h_\maj$, or not $h_\maj$. 
	In the first case, $L_S(h_\maj)\leq (2M + 2)e$, and so $L_D(h_\maj)\leq L_S(h_\maj) + 2\TV(\D, \hat\D)  \leq  (2M + 4)e$, and so $h_\maj\in \L(\D)$. In the second, for $h=\A(S)$, by construction, $L_S(h)=0$. Thus, $L_D(h)\leq L_S(h) + 2\TV(\D, \hat\D)  \leq  2e$, and again, $h\in \L(\D)$. So, $[\TV(\D, \hat\D) \leq e]\subseteq [\A(S)\in \L(\D)]$, and so the probability that $\A(S)$ is in $\L$ is at least $1-\delta$. The proof is finished by taking $n$ sufficiently big so that $e = e(n) \leq \varepsilon/(2M + 4)$.
\end{proof}

\bibliographystyle{plainnat}
%\bibliography{biblio}
%
%\bibliographystyle{plain}
\bibliography{library}

\end{document}